\documentclass[sigconf]{acmart}

\usepackage{amsmath}
\usepackage{bbm}

\usepackage{graphicx}
\usepackage{mathtools}

\usepackage{subcaption}
\usepackage{url}
\usepackage{multirow}
\usepackage{float}
\usepackage[ruled, vlined,algo2e,linesnumbered]{algorithm2e}
\usepackage{xcolor}
\usepackage{tikz}
\usetikzlibrary{fit,calc}

\colorlet{pink}{red!40}
\colorlet{blue}{cyan!60}

\newcommand{\matrixSym}[1]{\ensuremath{\mathbf{\uppercase{#1}}}}

\newcommand\footnoteref[1]{\protected@xdef\@thefnmark{\ref{#1}}\@footnotemark}
\newcommand{\RN}[1]{%
	\textup{\uppercase\expandafter{\romannumeral#1}}%
}
\newtheorem{definition}{Definition}

\usepackage{flushend}
\usepackage{balance}
\newtheorem{theorem}{Theorem}

\setcopyright{acmcopyright}
\copyrightyear{2020}
\acmYear{2020}
\acmDOI{10.1145/1122445.1122456}

\acmConference[Woodstock '20]{Woodstock '18: ACM Symposium on Neural
	Gaze Detection}{June 03--05, 2020}{Woodstock, NY}
\acmBooktitle{Woodstock '20: ACM Symposium on Neural Gaze Detection,
	June 03--05, 2020, Woodstock, NY}
\acmPrice{15.00}
\acmISBN{978-1-4503-XXXX-X/20/06}

\begin{document}

\sloppy

\citestyle{acmnumeric}

\title{Incorporating User's Preference into Attributed Graph Clustering}

\author{Wei Ye$^1$, Dominik Mautz$^2$, Christian B\"{o}hm$^2$, Ambuj Singh$^1$,Claudia Plant$^3$}
\affiliation{%
	\institution{$^1$University of California, Santa Barbara}
}
\affiliation{%
	\institution{$^2$Ludwig-Maximilians-Universit\"{a}t M\"{u}nchen, Munich, Germany}
}
\affiliation{%
	\institution{$^3$University of Vienna, Vienna, Austria}
}
\email{{weiye, ambuj}@cs.ucsb.edu}
\email{{mautz, boehm}@dbs.ifi.lmu.de}
\email{claudia.plant@univie.ac.at}

\renewcommand{\shortauthors}{We Ye et al.}

\begin{abstract}
	Graph clustering has been studied extensively on both plain graphs and attributed graphs. However, all these methods need to partition the whole graph to find cluster structures. Sometimes, based on domain knowledge, people may have information about a specific target region in the graph and only want to find a single cluster concentrated on this local region. Such a task is called local clustering. In contrast to global clustering, local clustering aims to find only one cluster that is concentrating on the given seed vertex (and also on the designated attributes for attributed graphs). Currently, very few methods can deal with this kind of task. To this end, we propose two quality measures for a local cluster: Graph Unimodality (\textsc{GU}) and Attribute Unimodality (\textsc{AU}). The former measures the homogeneity of the graph structure while the latter measures the homogeneity of the subspace that is composed of the designated attributes. We call their linear combination as \textsc{Compactness}. Further, we propose LOCLU to optimize the \textsc{Compactness} score. The local cluster detected by LOCLU concentrates on the region of interest, provides efficient information flow in the graph and exhibits a unimodal data distribution in the subspace of the designated attributes. 
\end{abstract}
	
		
\maketitle

\section{Introduction}\label{intro}
Data can be collected from multiple sources and modeled as attributed graphs (networks), in which vertices represent entities, edges represent their relations and attributes describe their own characteristics. For example, proteins in a protein-protein interaction network may be associated with gene expressions in addition to their interaction relations; users in a social network may be associated with individual attributes such as interests, residence and demographics in addition to their friendship relations. 

One of the major data mining tasks in graphs (networks) is the detection of clusters. Existing methods for cluster detection in attributed graphs can be divided into two categories, i.e., full space attributed graph clustering methods \cite{akoglu2012pics, DBLP:journals/pvldb/ZhouCY09} and subspace attributed graph clustering methods~\cite{DBLP:conf/icdm/GunnemannFRS13, chen2017generic, ye2017attributed}. The methods belonging to the first category treat all attributes equally important to the graph structure, while the methods belonging to the second category consider varying relevance of attributes to the graph structure. All these methods need to partition the whole graph to find cluster structures. However, based on domain knowledge, sometimes people may have information about a specific target region in the graph and are only interested in finding a cluster surrounding this local region. Such a task is called local cluster detection, which has aroused a great deal of attention in many applications, e.g., targeted ads, medicine, etc. Without considering scalability, one may think we can first use full space or subspace attributed graph clustering techniques and then return the cluster that contains the target region. However, it is hard to set the number of clusters in real-world graphs. And the cluster content depends on the chosen number of clusters. 

To deal with this task, several recent works~\cite{andersen2006communities, leskovec2008statistical} use short random walks starting from the target region to find the local cluster. Also, some approaches \cite{andersen2006local, kloster2014heat} focus on using the graph diffusion methods to find the local cluster. However, these methods are only suitable for detecting local clusters in plain graphs whose vertices have no attributes. Recently, FocusCO~\cite{DBLP:conf/kdd/PerozziASM14} has been proposed to find a local cluster of interest to users in attributed graphs. Given an examplar set, it first exploits a metric learning method to learn a projection vector that makes the vertex in the examplar set similar to each other in the projected attribute subspace, then updates the graph weight and finally performs the focused cluster extraction. FocusCO cannot infer the projection vector if the examplar set has only one vertex.

In this paper, given user's preference, i.e., the seed vertex and the designated attributes, we develop a method that can automatically find the vertices that are similar to the given seed vertex. The similarity is measured by the homogeneity both in the graph structure and the subspace that is composed of the designated attributes. To this end, we first propose \textsc{Compactness} to measure the unimodality\footnote{In this work, unimodality/unimodal and homogeneity/homogeneous can be used interchangeably.} of the clusters in attributed graphs.  \textsc{Compactness} is composed of two measures: Graph Unimodality (\textsc{GU}) and Attribute Unimodality (\textsc{AU}). \textsc{GU} measures the unimodality of the graph structure, and \textsc{AU} measures the unimodality of the subspace that is composed of the designated attributes. To consider both the graph structure and attributes, we first embed the graph structure into vector space. Then we consider the graph embedding vector as another designated attribute and apply the local clustering technique separately on each designated attribute. We call the procedure to find a local cluster as LOCLU.

Let us use a simple example to demonstrate our motivation. Figure~\ref{fig:introexample} shows an example social network, in which the vertices represent students in a middle school, the edges represent their friendship relations, and the attributes associated to each vertex are age (year), sport time (hour) per week, studying time (hour) per week and playing mobile game (M.G.) time (hour) per week. Given vertex 4 and the designated attribute M.G., the task is to find a local cluster around the vertex 4. (This task is of interest to mobile game producers.) Conventional diffusion-based local clustering method such as HK \cite{kloster2014heat} finds a cluster $\mathcal{C}_1=\{1, 3, 4, 5, 6, 7\}$. However, this cluster is not homogeneous in the subspace of the M.G. attribute. Compared with $\mathcal{C}_1$, the cluster $\mathcal{C}_2=\{1, 2, 3, 4\}$ is more local, which is concentrated on the vertex 4 and the M.G. attribute.
\begin{figure}[!htp]
	\centering
	\includegraphics[width=0.45\textwidth]{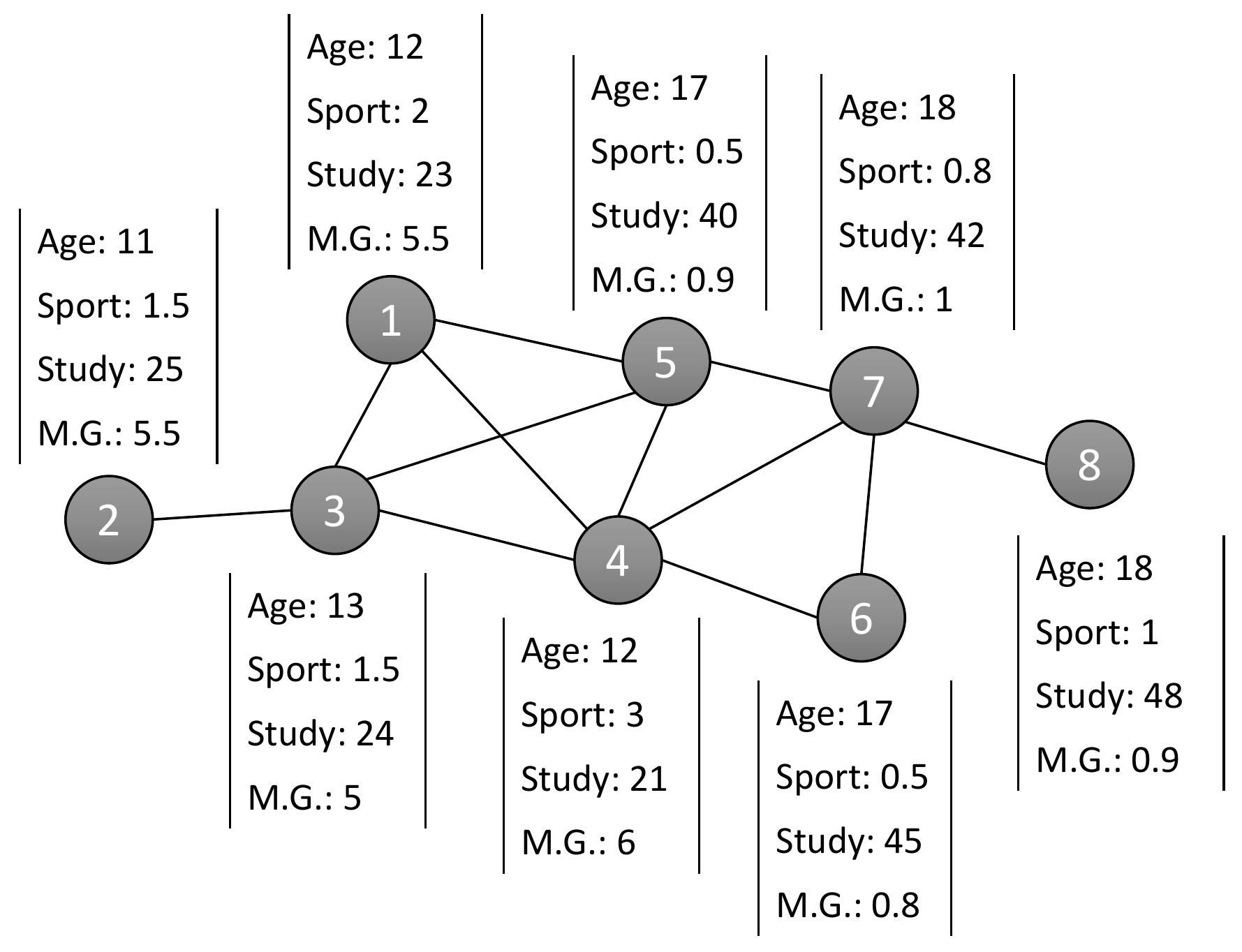}
	\caption{An example social network.}
	\label{fig:introexample}
\end{figure}

The main contributions are as follows:
\begin{itemize}
	\item We introduce the univariate statistic hypothesis test called Hartigans' dip test~\cite{hartigan1985dip} to a new user-centric problem setting: incorporating user's preference into attributed graph clustering.
	\item We propose \textsc{Compactness}, a new quality measure for clusters in attributed graphs. \textsc{Compactness} measures the homogeneity (unimodality) of both the graph structure and subspace that is composed of the designated attributes.
	\item We propose LOCLU to optimize the \textsc{Compactness} score.
	\item We demonstrate the effectiveness and efficiency of LOCLU by conducting experiments on both synthetic and real-world attributed graphs.
\end{itemize}

\section{Preliminaries}
\subsection{Notation}\label{sec:notation}
In this work, we use lower-case Roman letters (e.g.\ $a,b$) to denote scalars. We denote vectors (column) by boldface lower case letters (e.g.\ $\mathbf{x}$) and denote its $i$-th element by $\mathbf{x}(i)$. Matrices are denoted by boldface upper case letters (e.g.\ $\matrixSym{X}$). We denote entries in a matrix by non-bold lower case letters, such as $x_{i j}$. Row $i$ of matrix $\matrixSym{X}$ is denoted by $\mathbf{X}(i, :)$, column $j$ by $\mathbf{X}(:, j)$. A set is denoted by calligraphic capital letters (e.g. $\mathcal{S}$). An undirected attributed graph is denoted by $\mathsf{G}=(\mathcal{V},\mathcal{E}, \matrixSym{X})$, where $\mathcal{V}$ is a set of graph vertices with number $n=|\mathcal{V}|$ of vertices, $\mathcal{E}$ is a set of graph edges with number $e=|\mathcal{E}|$ of edges and $\matrixSym{X}\in \mathbb{R}^{n \times d}$ is a data matrix of attributes associated to vertices, where $d$ is the number of attributes. An adjacency matrix of vertices is denoted by $\mathbf{A}\in \mathbb{R}^{n\times n}$ with $a_{i j}=1 (i\neq j)$ and $a_{i j}=0 (i=j)$. The degree matrix $\mathbf{D}$ is a diagonal matrix associated with $\mathbf{A}$ with $d_{i i}=\sum_j a_{i j}$.  The random walk transition matrix $\mathbf{W}$ is defined as $\mathbf{D}^{-1}\mathbf{A}$. The Laplacian matrix is denoted as $\mathbf{L}=\mathbf{I}-\mathbf{W}$, where $\mathbf{I}$ is the identity matrix. An attributed graph cluster is a subset of vertices $\mathcal{C}\subseteq \mathcal{V}$ with attributes. The indicator function is denoted by $\mathbbm{1}(x)$.

\subsection{The Dip Test}
Before introducing the concept of the dip test, let us first clarify the definitions of unimodal distribution and multimodal distribution. In statistics, a unimodal distribution refers to a probability distribution that only has a single mode (i.e. peak). If a probability distribution has multiple modes, it is called multimodal distribution. From the behavior of the cumulative distribution function (CDF), unimodal distribution can also be defined as: if the CDF is convex for $x<m$ and concave for $x>m$ ($m$ is the mode), then the distribution is unimodal.

Now let us introduce a univariate statistic hypothesis test which is called Hartigans' dip test~\cite{hartigan1985dip} as follows:
 \begin{theorem}{~\cite{hartigan1985dip}}
 	Let $F(x)$ be a distribution function. Then $D(F)=2h$ ($h$ is the dip test value) only if there exists a nondecreasing function $G(x)$ such that for some $x_l\leq x_u$:
 	\begin{itemize}
 		\item $G(x)$ is the greatest convex minorant (g.c.m.) of $F(x)+h$ in $(-\infty, x_l)$.
 		
 		\item $G(x)$ has constant maximum slope in $[x_l, x_u]$ (modal interval).
 		
 		\item $G(x)$ is the least concave majorant (l.c.m.) of $F(x)-h$ in $(x_u, \infty)$.
 		
 		\item $h=\sup_{x\notin[x_l, x_u]}| F(x)-G(x)|\geq\sup_{x\in[x_l, x_u]}| F(x)-G(x)|$.
 	\end{itemize}
 \end{theorem}

 The g.c.m. of $F(x)$ in $(-\infty,x_l]$ is $\sup(L(x))$ for $x\leq x_l$, where the $\sup(\cdot)$ is taken over all functions $L(x)$ that are convex in $(-\infty,x_l]$ and nowhere greater than $F(x)$. The l.c.m. of $F(x)$ in $[x_u, \infty)$ is $\inf(L(x))$ for $x\geq x_u$, where the $\inf(\cdot)$ is taken over all functions $L(x)$ that are concave in $[x_u, \infty)$ and nowhere less than $F(x)$.

The dip test is the infimum among the supremum computed between the cumulative distribution function (CDF) of $F(x)$ and the CDF of $G(x)$ from the set of unimodal distributions. The dip test measures the departure of $F(x)$ from unimodality. As pointed out in~\cite{hartigan1985dip}, the class of uniform distributions is the most suitable for the null hypothesis, because their dip test values are stochastically larger than those of other unimodal distributions. Note that the higher the dip test value, the more multimodal the distribution. Also note that the dip test value is in the range $[0, 0.25)$~\cite{maurus2016skinny}.

 Let us use Figure~\ref{fig:demo} to demonstrate the main idea behind the dip test. Figure~\ref{fig:demo}(a) shows the histogram of the $x$-axis projection of the data shown in Figure~\ref{fig:example}(a). The blue curve ($F(x)$) in Figure~\ref{fig:demo}(b) is the CDF of the $x$-axis projection of the data. Note that the histogram is for visual comparison only. The dip test only needs $F(x)$. To measure the unimodality of $F(x)$, the dip test tries to fit a piecewise-linear function $G(x)$ onto it. Then, the twice of the dip test value $2h$ is defined as the maximum achievable vertical offset for two copies of $F(x)$ (the red and magenta curves, i.e., $F(x)+h$ and $F(x)-h$ in Figure~\ref{fig:demo}(b)) such that $G(x)$ does not violate its unimodal rules (convex up to the modal interval (the shade area in Figure \ref{fig:demo}(b)) and concave after it). The farther $F(x)$ strays from unimodality, the larger the required offset between the two copies of $F(x)$. For more details, please refer to~\cite{hartigan1985dip,krause2005multimodal,maurus2016skinny}.
\begin{figure}[!htb]
	\centering
	\begin{subfigure}{.23\textwidth}
		\includegraphics[width=\textwidth]{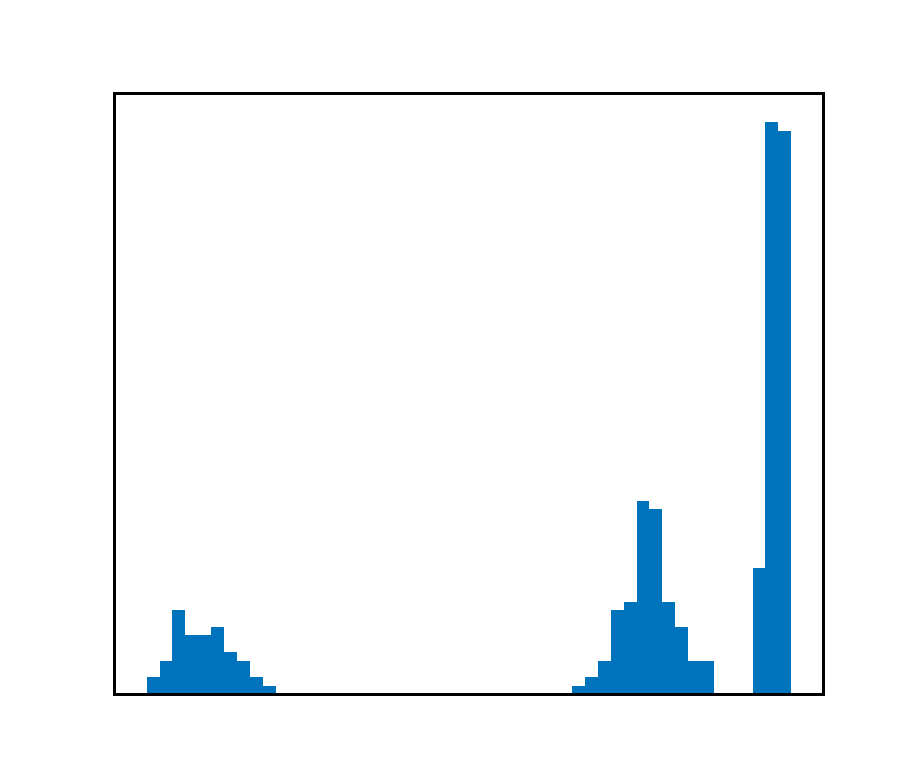}
		\caption{Histogram.} 
	\end{subfigure}
	\centering
	\begin{subfigure}{.23\textwidth}
		\includegraphics[width=\textwidth]{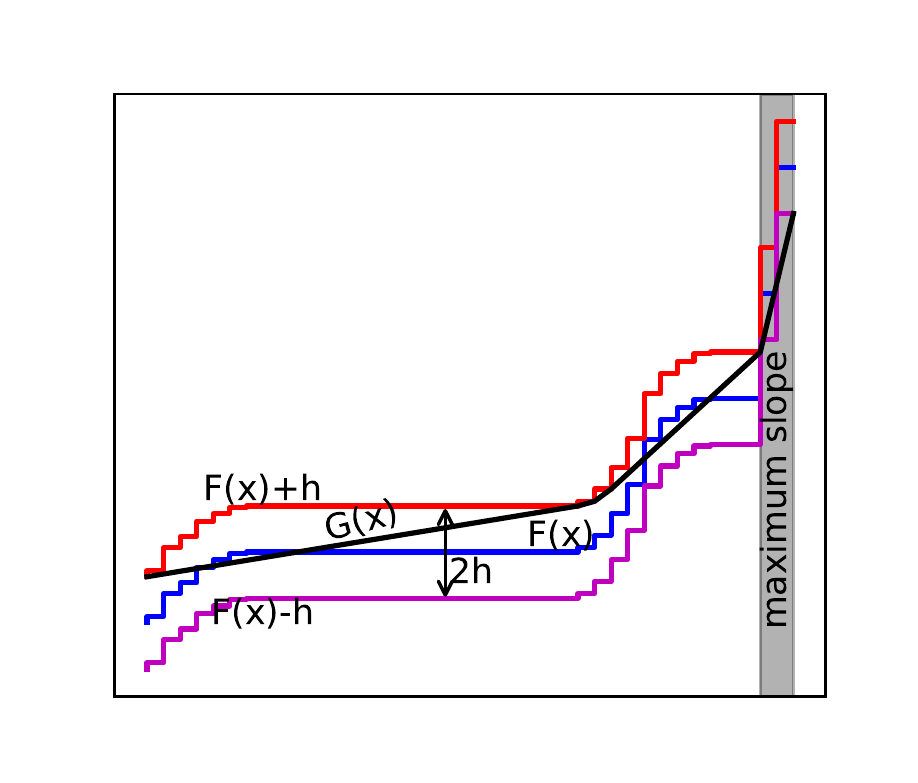}
		\caption{Dip solution.} 
	\end{subfigure}
	\caption{The demonstration of the dip test.}
	\label{fig:demo}
\end{figure}

The $p$-value for the dip test is then computed by comparing $D(F(x))$ with $D(G(x)^{(q)})$ $b$ times, each time with a different $n$ observations from $G(x)$, and the proportion $\sum_{1\leq q\leq b}\mathbbm{1}(D(F(x))\leq D(G(x)^{(q)}))/b$ is the $p$-value. If the $p$-value is greater than a significance level $\alpha$, say 0.05, the null hypothesis $H_0$ that $F(x)$ is unimodal is accepted. Otherwise $H_0$ is rejected in favor of the alternative hypothesis $H_1$ that $F(x)$ is multimodal.

\section{Method LOCLU}
\subsection{Objective Function}
The problem of incorporating user's preference into attributed graph clustering can be defined as follows:
\newtheorem*{problemLOCLU}{\textbf{Incorporating User's Preference into Attributed Graph Clustering}}
\begin{problemLOCLU}
	Given an attributed graph $\mathsf{G}=(\mathcal{V},\mathcal{E}, \matrixSym{X})$, a seed vertex $v_q$, and the indexes of the designated attributes $\mathcal{I}=\{a_1,a_2,\cdots,a_u\}$, find a cluster $\mathcal{C}=\{v_1,v_2,\cdots,v_q,\cdots\}$ around the seed vertex $v_q$, such that the cluster $\mathcal{C}$ is not only unimodal in the graph structure but also in the subspace that is composed of the designated attributes.
\end{problemLOCLU}

In order to find the local cluster that satisfies the definition, we need to consider the information from both the graph structure and attributes. Firstly, we consider the graph structure. We propose Graph Unimodality (\textsc{GU}) to measure the unimodality of the graph structure.
\begin{definition}[Graph Unimodality]
	For a local cluster $\mathcal{C}=\{v_1,v_2,\cdots,v_q,\cdots\}$ around the given seed vertex $v_q$, its graph unimodality is defined as:
	\begin{equation}
	\label{equ:gu}
	\textsc{GU}(\mathcal{C}) = \frac{1}{r}\sum_{i=1}^rD\left(F\left(\mathbf{e}_i(\mathcal{C})\right)\right)
	\end{equation}
	where $r$ is the dimension of the graph embedding $\matrixSym{E}$, $\mathbf{e}_i$ is the $i$-th embedding vector, and $F\left(\mathbf{e}_i(\mathcal{C})\right)$ is the CDF of $\mathbf{e}_i(\mathcal{C})$.
\end{definition}

Graph Unimodality (\textsc{GU}) measures the unimodality of the graph structure in a detected local cluster. The lower the GU value, the more unimodal a local cluster in the graph structure. We use the spectral embedding method, especially normalized cut (NCut)~\cite{shi2000normalized}, to find the embedding $\matrixSym{E}$ of the graph structure. The definition of the NCut is:

\begin{equation}
\label{eqn:Ncut}
\mbox{NCut}(\mathcal{C})=\frac{\mbox{cut}(\mathcal{C},\overline{\mathcal{C}})}{\mbox{vol}(\mathcal{C})}
\end{equation}
where $\mbox{cut}(\mathcal{C},\overline{\mathcal{C}})=\sum_{v_i\in\mathcal{C},v_j\in\overline{\mathcal{C}}}a_{i j}$ and $\mbox{vol}(\mathcal{C})=\sum_{v_i\in\mathcal{C},v_j\in\mathcal{V}}a_{i j}$.

Equation (~\ref{eqn:Ncut}) can be equivalently rewritten as (for a more detailed explanation, please refer to~\cite{von2007tutorial}):
\begin{equation}
\label{eqn:relaxedNcut}
\begin{aligned}
&\mbox{NCut}(\mathcal{C})=\mathbf{e}^\intercal\mathbf{L}\mathbf{e} \\ 
s.t.  \; &\mathbf{e}^\intercal\mathbf{D}\mathbf{e}=\mbox{vol}(\mathsf{G})\\
&\mathbf{D}\mathbf{e}\perp\textbf{1}
\end{aligned}
\end{equation}
where $\mathbf{e}$ is the cluster indicator vector (embedding vector) and $\mathbf{e}^\intercal\mathbf{L}\mathbf{e}$ is the cost of the cut and $\textbf{1}$ is a constant vector whose entries are all 1. Note that finding the optimal solution is known to be NP-hard \cite{wagner1993between} when the values of $\mathbf{e}$ are constrained to $\{1,-1\}$. But if we relax the objective function to allow it take values in $\mathbb{R}$, a near optimal partition of the graph $\mathsf{G}$ can be derived from the eigenvector having the second smallest eigenvalue of $\mathbf{L}$. More generally, embedding $\matrixSym{E}$ that is composed of $k$ eigenvectors with the $k$ smallest eigenvalues partition the graph into $k$ subgraphs with near optimal normalized cut value.

Secondly, we consider the attributes. We propose Attribute Unimodality (\textsc{AU}) to measure the unimodality of the subspace that is composed of the designated attributes.
\begin{definition}[Attribute Unimodality]
	For a local cluster $\mathcal{C}=\{v_1,v_2,\cdots,v_q,\cdots\}$ around the given seed vertex $v_q$, its attribute unimodality is defined as:
	\begin{equation}
	\label{equ:au}
	\textsc{AU}(\mathcal{C}) = \frac{1}{u}\sum_{i=1}^uD\left(F\left(\mathbf{x}_i(\mathcal{C})\right)\right)
	\end{equation}
	where $u$ is the number of the designated attributes, $\mathbf{x}_i$ is the $a_i$-th designated attribute, and $F\left(\mathbf{x}_i(\mathcal{C})\right)$ is the CDF of $\mathbf{x}_i(\mathcal{C})$.
\end{definition}

Attribute Unimodality (\textsc{AU}) measures the unimodality of the subspace that is composed of the designated attributes in a detected local cluster. The lower the AU value, the more unimodal a local cluster in the subspace. 

To measure the unimodality of both the graph structure and attributes of a local cluster, our objective function integrates both \textsc{GU} and \textsc{AU} into one framework, which is called \textsc{Compactness}:
\begin{equation}
\label{equ:obj}
\textsc{Compactness}(\mathcal{C}) = \mbox{GU}(\mathcal{C}) + \mbox{AU}(\mathcal{C})
\end{equation}

In the following, we will elaborate the optimization method LOCLU. It employs a dip test based local clustering technique on the embedding of the graph structure and the designated attributes. The detected local cluster is unimodal both in the graph structure and the designated attributes.

\subsection{Optimizing Attribute Unimodality}\label{sec:unimodal_clustering}
There are many clustering techniques for the numerical attributes, e.g., $k$-means, EM, DBSCAN~\cite{ester1996density}, etc. One possible idea is inputing proper parameters (such as the number of clusters for $k$-means, and MinPts and $\epsilon$ for DBSCAN) to the clustering techniques and letting them return the cluster that includes the given seed vertex. However, the disadvantage is that the parameters are difficult to set. The cluster assignments change with different numbers of clusters. Moreover, many clustering techniques need to partition the whole dataset, which is very time- and resource-consuming. Alternatively, we can consider the attributes and graph structure simultaneously and apply some attributed graph clustering technique. The problems we face are the same as described above. In this paper, we would like to develop a local clustering technique. The technique  does not require the number of clusters $k$, which is difficult to set in the real world datasets.

Note that the dip test returns a modal interval, in which the distribution of data is unimodal. Our idea is to employ the modal interval to find a local cluster around the given seed vertex. Our perspective is that data points in the modal interval belong to one cluster. We cluster vertices according to their positions and the position of the modal interval. Thus, other statistical tests that do not return the modal interval cannot be adopted. We use Figure~\ref{fig:example} to elaborate the main idea of our local clustering technique. Assume that Figure~\ref{fig:example}(a) shows two numerical attributes $x$ and $y$ associated with a local graph cluster. The two attributes have three clusters inside. The purple and blue clusters follow Gaussian distributions. The orange cluster follows a uniform distribution. The big blue dot in the blue cluster reveals the numerical values of the given seed vertex and we want to find a local cluster concentrating on this big blue dot and the $x$-axis. 

We input the $x$-axis projection of the data into the dip test and it returns a $p$-value and a modal interval $[x_l^{(1)}, x_u^{(1)}]$. In this case, the $p$-value is 0 that is below the significance level $\alpha=0.05$, which means the $x$-axis projection of the data is multimodal. If the given seed point is on the left side of $x_l^{(1)}$, we remove the data points that situate on the right side of $x_l^{(1)}$ and dip over the $x$-axis projection of the remaining data. If the given seed point is on the right side of $x_u^{(1)}$, we remove the data points that situate on the left side of $x_u^{(1)}$ and dip over the $x$-axis projection of the remaining data. Otherwise, we dip over the $x$-axis projection of the data situated in the modal interval $[x_l^{(1)}, x_u^{(1)}]$. We repeat the process until the $x$-axis projection of the remaining data that contains the given seed point is unimodal.

In our case, since $[x_l^{(1)}, x_u^{(1)}]$ does not contain the big blue dot, we remove the data points that situate on the right side of $x_l^{(1)}$ and continue to dip over the $x$-axis projection of the remaining data. The dip test on the $x$-axis projection of the remaining data (shown in Figure \ref{fig:example}(b)) returns a $p$-value of 0, which indicates that the remaining data is still multimodal. Because the given seed point is within the new modal interval $[x_l^{(2)}, x_u^{(2)}]$, we extract the data points situated in this modal interval and dip over their $x$-axis projection. Since the $p$-value returned by the dip test is 0.937 that is greater than the significance level $\alpha=0.05$, which means the $x$-axis projection of the data points (shown in Figure \ref{fig:example}(c)) is unimodal, we terminate the recursive process and return the found local cluster (shown in Figure~\ref{fig:example}(c)).
\begin{figure*}[!htb]
	\centering
	\begin{subfigure}{.3\textwidth}
		\includegraphics[width=\textwidth]{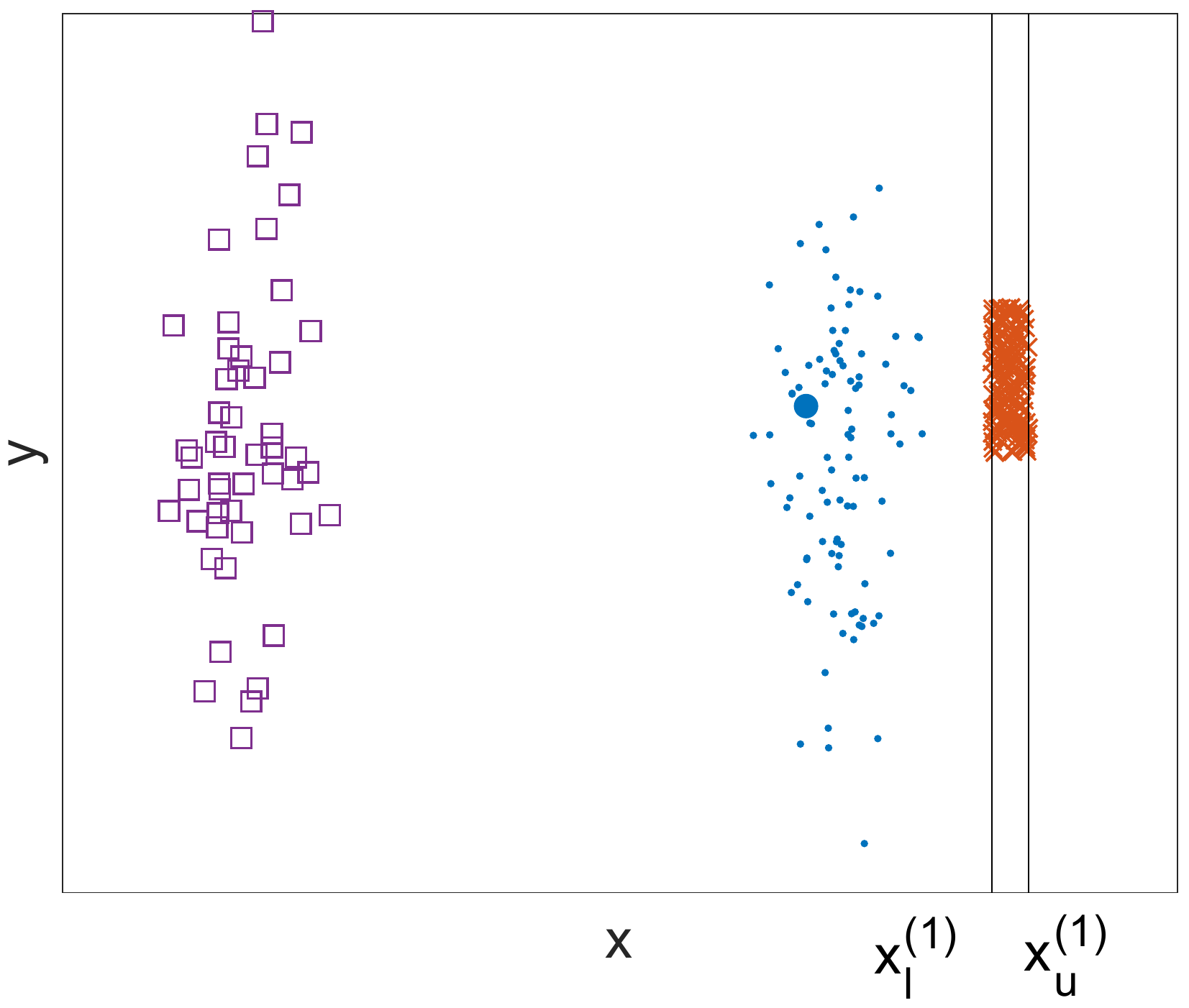}
		\caption{The $x$-axis projection is multimodal with the modal interval $[x_l^{(1)}, x_u^{(1)}]$.} 
	\end{subfigure}
    \hfill
	\centering
	\begin{subfigure}{.3\textwidth}
		\includegraphics[width=\textwidth]{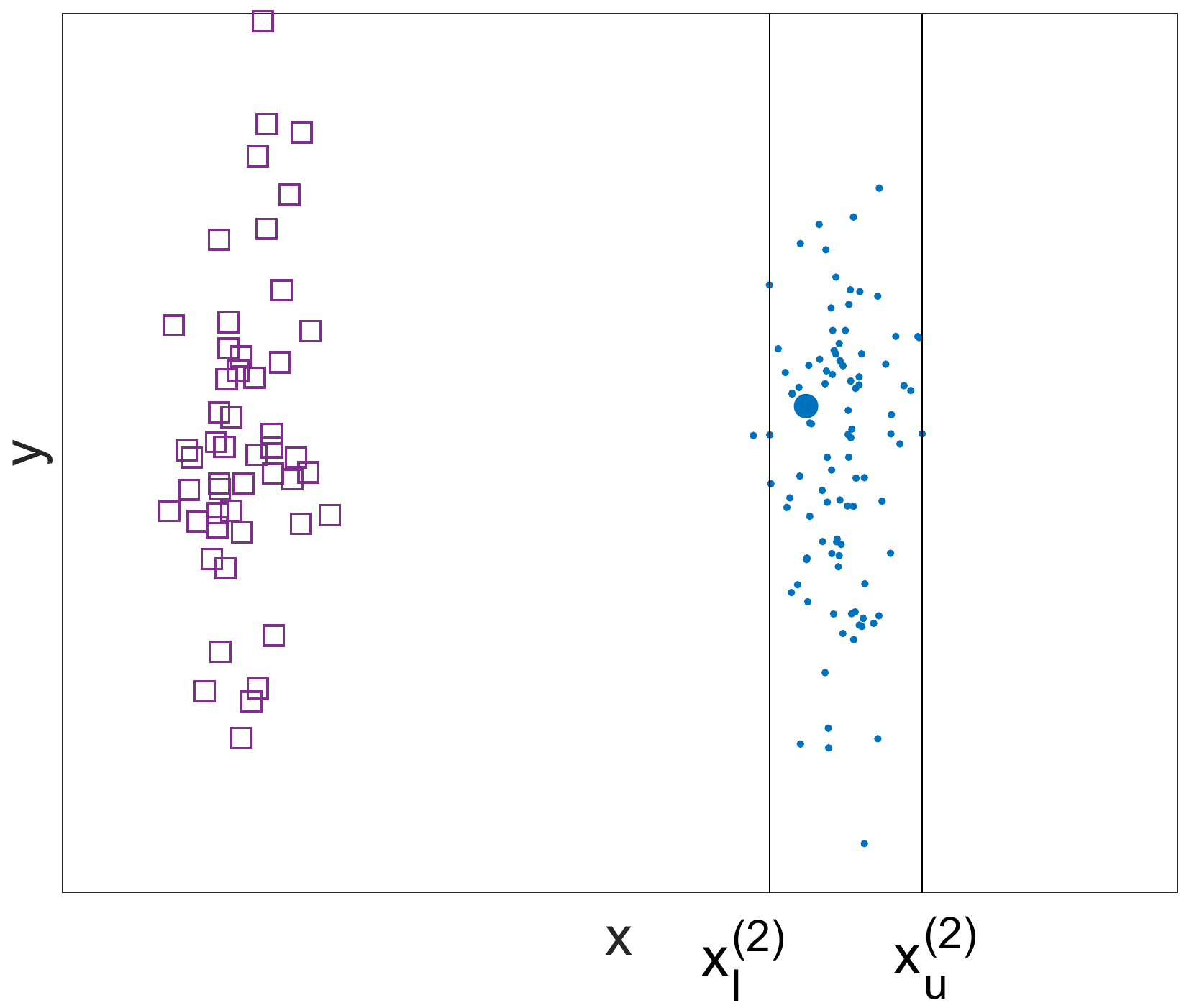}
		\caption{The $x$-axis projection is multimodal with the modal interval $[x_l^{(2)}, x_u^{(2)}]$.} 
	\end{subfigure}
    \hfill
	\centering
	\begin{subfigure}{.3\textwidth}
		\includegraphics[width=\textwidth]{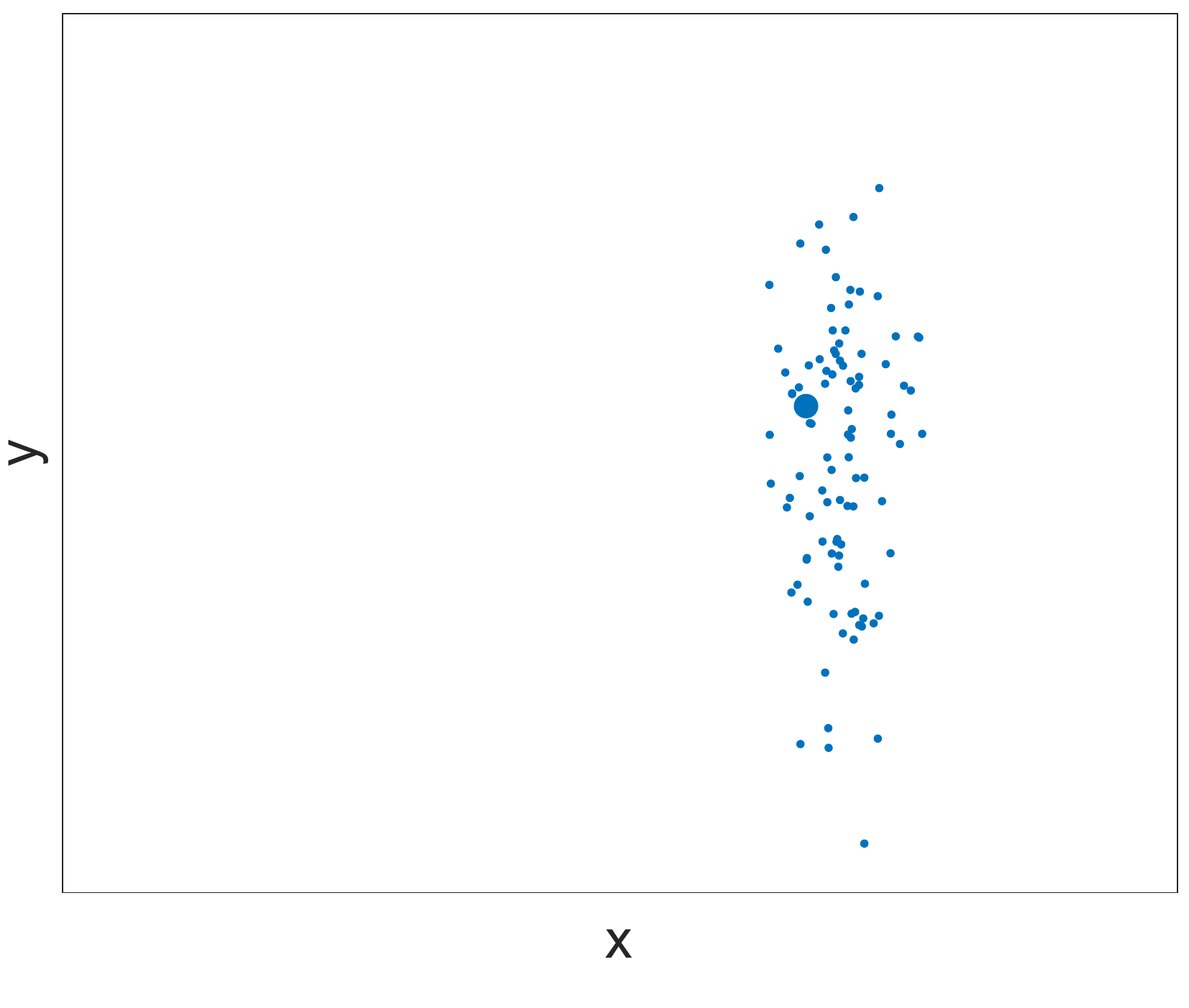}
		\caption{The $x$-axis projection is unimodal and the local clustering stops.} 
	\end{subfigure}
	\caption{The demonstration for the local clustering technique. Assume that the attributes $x$ and $y$ are associated with a local graph cluster. Given the seed vertex (the big blue dot) and the index of the designated attribute ($x$), we want to find a local cluster around the blue dot, which is unimodal in the subspace that is composed of the designated attribute $x$.}
	\label{fig:example}
\end{figure*}

In the above, we recursively dip over the $x$-axis projection of the data to find the local cluster. If given the indexes of attributes $\mathcal{I}=\{a_1,a_2,\cdots,a_u\}$, we will first compute the dip test value of each attribute and then dip over the attributes according to their dip test values, from the highest to the lowest. In this way, the most multimodal attribute will be firstly explored. The insight is that the directions which depart the most from unimodality are promising for clustering.
\begin{theorem}\label{the:uni}
	For a random variable $X=[x_1,x_2,\ldots,x_n]$, the local clustering method will find a unimodal cluster whose values are in the interval $[x_l, x_u]$. The probability distribution function (PDF) of the data points in the interval $[x_i, x_j]$ ($\forall i, j,  x_l\leq x_i<x_j\leq x_u$) is unimodal.
\end{theorem}

\begin{proof}
	If the PDF of the data points in the interval $[x_i, x_j]$ ($\forall i, j, x_l\leq x_i<x_j\leq x_u$) is multimodal, the PDF of the data points in the interval $[x_l, x_u]$ should also be multimodal. Thus, the cluster in the interval $[x_l, x_u]$ is not unimodal and the local clustering method will continue dipping over the interval $[x_l, x_u]$ until the PDF of the remaining data points is unimodal.
\end{proof}

\subsection{Optimizing Graph Unimodality}
A good partition of the graph structure is achieved by using the local clustering method in the embedding vector space. In this work, the graph embedding just contains one eigenvector. Each eigenvector bisects the graph into two clusters. We can consider the cluster that contains the seed vertex as the required local cluster. We can first compute the second smallest eigenvector $\mathbf{e}_2$ of the graph Laplacian matrix $\mathbf{L}$. Then, we use the proposed local clustering technique on it to find a local cluster that contains the seed vertex $v_q$. For large-scale graphs, the eigen-decomposition of $\mathbf{L}$ is $\mathcal{O}(n^3)$ which is impractical. Instead, we use the power iteration method~\cite{lin2010power} to compute an approximate eigenvector. 

The power iteration (PI) is a fast method to compute the dominant eigenvector of a matrix. Note that the $k$ largest eigenvector of $\mathbf{W}$ are also the $k$ smallest eigenvector of $\mathbf{L}$. The power iteration method starts with a randomly generated vector $\mathbf{v}^0$ and iteratively updates as follows,
\begin{equation}
\label{fig:update}
\mathbf{v}^t=\frac{\mathbf{W}\mathbf{v}^{t-1}}{\lVert\mathbf{W}\mathbf{v}^{t-1} \rVert_1}
\end{equation}

Suppose $\mathbf{W}$ has eigenvectors (embedding vectors) $\mathbf{E}=[\mathbf{e}_1;\mathbf{e}_2;\ldots;\mathbf{e}_n]$ with eigenvalues $\mathbf{\Lambda}=[\lambda_1,\lambda_2,\ldots,\lambda_n]$, where $\lambda_1=1$ and $\mathbf{e}_1$ is constant. We have $\mathbf{W}\mathbf{E}=\mathbf{\Lambda}\mathbf{E}$ and in general $\mathbf{W}^t\mathbf{E}=\mathbf{\Lambda}^t\mathbf{E}$. When ignoring renormalization, Equation (~\ref{fig:update}) can be written as 
\begin{equation}
\label{fig:updateT}
\begin{split}
\mathbf{v}^t&=\mathbf{W}\mathbf{v}^{t-1}=\mathbf{W}^2\mathbf{v}^{t-2}=\cdots =\mathbf{W}^t\mathbf{v}^0  \\
&=\mathbf{W}^t\left( c_1\mathbf{e}_1+c_2\mathbf{e}_2+\cdots +c_n\mathbf{e}_n\right) \\
&=c_1\mathbf{W}^t\mathbf{e}_1+c_2\mathbf{W}^t\mathbf{e}_2+\cdots +c_n\mathbf{W}^t\mathbf{e}_n \\
&=c_1\lambda_1^t\mathbf{e}_1+c_2\lambda_2^t\mathbf{e}_2+\cdots +c_n\lambda_n^t\mathbf{e}_n
\end{split}
\end{equation}
where $\mathbf{v}^0$ can be denoted by $c_1\mathbf{e}_1+c_2\mathbf{e}_2+\cdots +c_n\mathbf{e}_n$ which is a linear combination of all the original orthonormal eigenvectors. Since the orthonormal eigenvectors form a basis for $\mathbb{R}^n$, any vector can be expanded by them. 

From Equation (~\ref{fig:updateT}), we have the following:
\begin{equation}
\frac{\mathbf{v}^t}{c_1\lambda_1^t}=\mathbf{e}_1+\frac{c_2}{c_1}\left(\frac{\lambda_2}{\lambda_1} \right)^t \mathbf{e}_2+\cdots +\frac{c_n}{c_1}\left(\frac{\lambda_n}{\lambda_1} \right)^t \mathbf{e}_n
\end{equation}

So the convergence rate of PI towards the dominant eigenvector $\mathbf{e}_1$ depends on the significant terms $\left(\frac{\lambda_i}{\lambda_1} \right)^t (2\leqslant i \leqslant n)$. If we let the power iteration method run long enough, it will converge to the dominant eigenvector $\mathbf{e}_1$ which is of little use in clustering. If we define the velocity at $t$ to be the vector $\boldsymbol{\delta}^t=\mathbf{v}^t-\mathbf{v}^{t-1}$ and define the acceleration at $t$ to be the vector $\boldsymbol{\epsilon}^t=\boldsymbol{\delta}^t-\boldsymbol{\delta}^{t-1}$, we can stop the power iteration when $\lVert\boldsymbol{\epsilon}^t\rVert_{max}$ is below a threshold $\hat{\epsilon}$. We use $\mathbf{v}^t$ as the graph embedding vector. Figure \ref{fig:compare_pdf} shows $\mathbf{e}_2$ and $\mathbf{v}^t$ of the graph Laplacian matrix of the data shown in Figure \ref{fig:example}. Compared with the PDF of $\mathbf{e}_2$, the PDF of $\mathbf{v}^t$ is more multimodal and thus is more promising for clustering.
\begin{figure}[!htb]
	\centering
	\begin{subfigure}{.23\textwidth}
		\includegraphics[width=\textwidth]{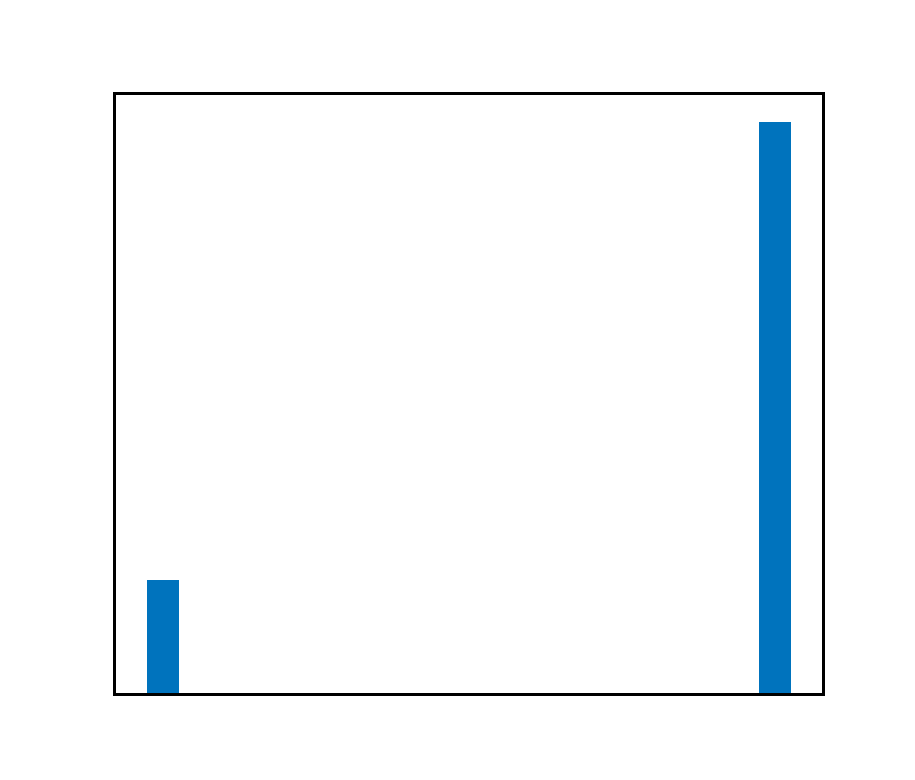}
		\caption{Histogram of $\mathbf{e}_2$.} 
	\end{subfigure}
	\centering
	\begin{subfigure}{.23\textwidth}
		\includegraphics[width=\textwidth]{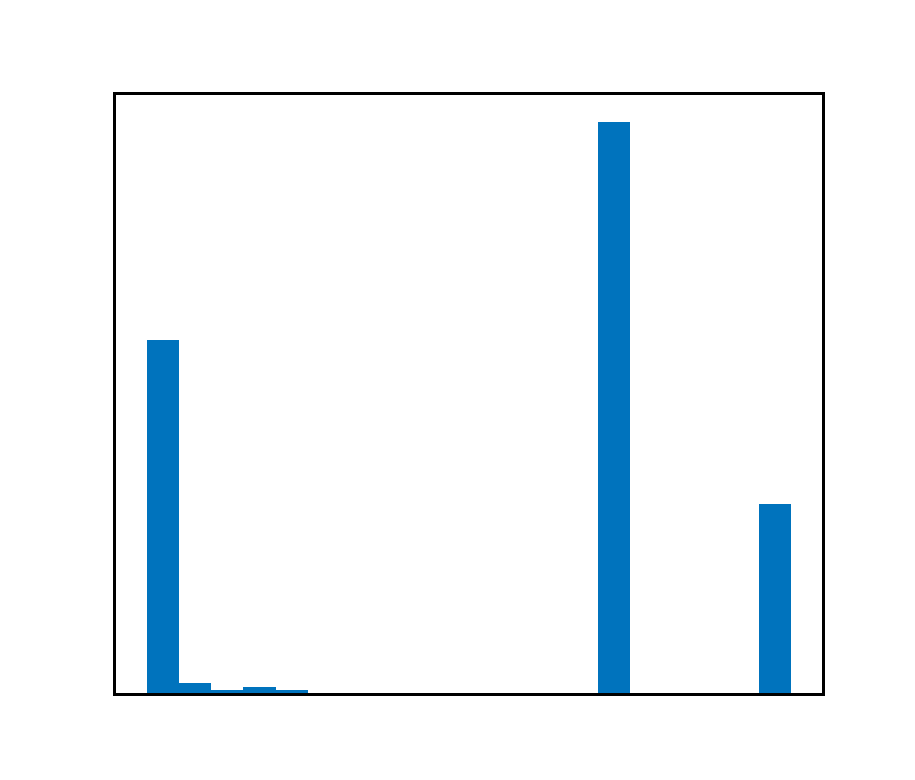}
		\caption{Histogram of $\mathbf{v}^t$.} 
	\end{subfigure}
	\caption{Comparison between the PDFs of $\mathbf{e}_2$ and $\mathbf{v}^t$.}
	\label{fig:compare_pdf}
\end{figure}

\subsection{Implematation Details and Analysis}
We can consider the embedding vector $\mathbf{v}^t$ as another designated attribute and perform local clustering on it. Let us elaborate the algorithmic details of LOCLU whose pseudo-code is given in Algorithm 1. Line 3 computes the random walk transition matrix, which costs $\mathcal{O}(e)$ where $e$ represents the number of edges in the graph. Line 4 initializes the starting vector for the power iteration method. Lines 5--9 use the power iteration method to compute an embedding vector for the graph structure. The power iteration method is guaranteed to converge (please refer to~\cite{lin2010power}). The time complexity for the power iteration method is $\mathcal{O}(e)$~\cite{lin2010very}. Line 10 considers the embedding vector $\mathbf{v}^t$ as another designated attribute and concatenates it to the data matrix $\mathbf{X}$. Lines 11--16 perform local clustering separately on the designated attributes. Lines 11--13 compute the dip test values and $p$-values for the designated attributes. At lines 15--16, we dip over the designated attributes according to their dip test values, from the highest to the lowest. Line 16 uses the local clustering method to find a local cluster around the given seed vertex on each designated attribute. The time complexity of the dip test at line 12 is $\mathcal{O}(n)$~\cite{hartigan1985dip}. Note that the dip test method first sorts the input data, which costs $\mathcal{O}(n\cdot \log(n))$ time. Thus, the time complexity of lines 11--14 in Algorithm 1 is $\mathcal{O}(u\cdot n\cdot \log(n))$, where $u$ is the number of the designated attributes and $n$ is the number of vertices.

The local clustering method is given in Algorithm 2. It recursively dips over the designated attribute until finding the local cluster around the given seed vertex. Line 3 dips over the designated attribute. At lines 4--9, if the given seed vertex does not belong to the current modal interval, we extract the vertices whose attribute values are on the left side of $x_l$ (line 5) or on the right side of $x_u$ (line 7) and update the cluster $\mathcal{C}$; at line 9, if the given seed vertex belongs to the current modal interval, we update the cluster $\mathcal{C}$ with the vertices whose attribute values are inside the modal interval. The time complexity of the dip test at line 3 is $\mathcal{O}(n\cdot \log(n))$. Thus, the time complexity of the local clustering procedure is bounded by $\mathcal{O}(k\cdot n\cdot \log(n))$, where $k\ll n$ is the number of modes in the data. We remark that the local clustering method is guaranteed to converge. The worst case is that each vertex is a mode and local clustering method finds the given seed vertex as the local cluster, which will lead to the termination of the dip test. Thus, LOCLU is also guaranteed to converge. The time complexity of LOCLU is $\mathcal{O}\left((u+k)\cdot n\cdot \log(n)+e\right)$.

\begin{theorem}
	The local cluster detected by LOCLU is unimodal in each designated attribute and the graph structure.
\end{theorem}

\begin{proof}
	For the data matrix $\mathbf{X}=[\mathbf{X}, \mathbf{v}^t]$ at line 10 in Algorithm 1, we first apply the local clustering method on the attribute with the highest dip test value and it will find a unimodal cluster in the interval $[x_{l_1}^{(1)}, x_{u_1}^{(1)}]$. Then we apply it on the attribute with the second largest dip test value and it will find a unimodal cluster in the interval $[x_{l_2}^{(2)}, x_{u_2}^{(2)}]$, where $l_1\leq l_2<u_2\leq u_1$. From Theorem \ref{the:uni}, we know that the data points in the interval $[x_{l_2}^{(1)}, x_{u_2}^{(1)}]$ of the attribute with the highest dip test value remains unimodal. If we apply the local clustering method on each column of $\mathbf{X}$, from the highest to the lowest with respect to their dip test values, the final local cluster detected by LOCLU is unimodal in each designated attribute and the graph structure.
\end{proof}

\begin{algorithm2e}
	\SetKw{KwRn}{randn}
	\SetKw{KwOr}{or}
	\SetKw{KwDip}{DipTest}
	\SetKw{KwSort}{sort}
	\SetKwFunction{UC}{LocalClustering}
	\SetKwFunction{Pr}{Prune}
	\SetKwFunction{MPP}{MultiProPursuit}
	\SetKwFunction{recurseDip}{GCD}
	\KwIn{Adjacency matrix $\mathbf{A}$, data matrix $\mathbf{X} \in \mathbb{R}^{n \times d}$, the seed vertex index $q$, the indexes of the designated attributes $\mathcal{I}=\{a_1,a_2,\cdots,a_u\}$}
	\KwOut{Local cluster $\mathcal{C}$}
	$\hat{\epsilon}\leftarrow 0.001$, $t\leftarrow 0$, \upshape iter$\leftarrow 1000$\;
	$\mathcal{C}\leftarrow [1:n]$\tcc*[r]{$\mathcal{C}$ contains the indexes of vertices.}
    compute the random walk transition matrix $\mathbf{W}$\;
    $\mathbf{v}^0\leftarrow$ \KwRn $(n,1)$\;
    \tcc{power iteration}
    \Repeat{$\lVert\boldsymbol{\delta}^{t+1}-\boldsymbol{\delta}^t\rVert_{\max}\leq\hat{\epsilon}$ \KwOr $t\geq$ \upshape iter}{
    	$\mathbf{v}^{t+1}\leftarrow \frac{\mathbf{W}\mathbf{v}^t}{\lVert\mathbf{W}\mathbf{v}^t \rVert_1}$\;
    	$\boldsymbol{\delta}^{t+1}\leftarrow \lvert\mathbf{v}^{t+1}-\mathbf{v}^t\rvert$\;
    	$t\leftarrow t+1$\;
    }
    $\mathbf{X}\leftarrow [\mathbf{X}, \mathbf{v}^t]$, $a_{u+1}\leftarrow d+1 $, $\mathcal{I}\leftarrow \mathcal{I}\cup a_{u+1}$\tcc*[r]{$[\cdot, \cdot]$ means concatenation.}
    \tcc{Perform local clustering separately on the embedding vector of the graph structure and designated attributes.}
	\For{$i \leftarrow1\hspace{0.5em}\KwTo\hspace{0.5em}u+1\hspace{0.5em}$}{
		$[\mbox{dip}, p\mbox{-value}, \mathbf{t},  x_l, x_u]\leftarrow$ \KwDip($\mathbf{X}(:,a_i)$)\;
		$\mathbf{d}(i)\leftarrow \mbox{dip}$\;
	}
	$[\mathbf{d},\mathbf{s}]\leftarrow \KwSort(\mathbf{d})$\tcc*[r]{descending sort, $\mathbf{s}$ contains the indexes of attributes     sorted by their dip test values.}
	\For{$i \leftarrow 1 \hspace{0.5em}\KwTo\hspace{0.5em} u+1\hspace{0.5em}$}{
			$\mathcal{C}\leftarrow$\UC{$\mathcal{C}$,$\mathbf{X}$,$q$,$\mathbf{s}(i)$}\;
    }
    \Return{$\mathcal{C}$}\;
	\caption{LOCLU}
	\label{alg:loclu}
\end{algorithm2e}

\begin{algorithm2e}
	\SetKw{KwDip}{DipTest}
	\SetKw{KwSort}{sort}
	\KwIn{Cluster $\mathcal{C}$, data matrix $\mathbf{X}$, the seed vertex index $q$, the index $s$ of the designated attribute}
	\KwOut{Local cluster $\mathcal{C}$}
	\Repeat{$p\mbox{-value}>0.05$}{
		$\mathbf{x}\leftarrow\mathbf{X}(\mathcal{C},s)$\;
		$[\mbox{dip}, p\mbox{-value}, x_l, x_u]\leftarrow$ \KwDip($\mathbf{x}$)\;
		\uIf{$\mathbf{x}(q)<x_l$}{
			$\mathcal{C}\leftarrow \{\sigma_1,\sigma_2,\ldots\}$\tcc*[r]{$\left\lbrace \mathbf{x}(\sigma_1),\mathbf{x}(\sigma_2),\ldots\right\rbrace <x_l$}
		}
		\uElseIf{$\mathbf{x}(q)>x_u$}{
			$\mathcal{C}\leftarrow \{\sigma_1,\sigma_2,\ldots\}$\tcc*[r]{$\left\lbrace \mathbf{x}(\sigma_1),\mathbf{x}(\sigma_2),\ldots\right\rbrace >x_u$}
		}
		\Else{
			$\mathcal{C}\leftarrow \{\sigma_1,\sigma_2,\ldots\}$\tcc*[r]{$\left\lbrace \mathbf{x}(\sigma_1),\mathbf{x}(\sigma_2),\ldots\right\rbrace \in[x_l, x_u]$}
		}
	}
	\Return{$\mathcal{C}$}\;
	\caption{LocalClustering}
	\label{alg:uc}
\end{algorithm2e}

\section{Experimental Evaluation}\label{experiments}
\subsection{Experiment Settings}
We thoroughly evaluate LOCLU on cluster quality and runtime using both synthetic and real-world attributed graphs. We compare LOCLU with baseline methods whose descriptions are as follows:
\begin{itemize}
	\item FocusCO~\cite{DBLP:conf/kdd/PerozziASM14} identifies the relevance of vertex attributes that makes the user-provided examplar vertices similar to each other. Then it reweighs the graph edges and extracts the focused cluster.
\item SG-Pursuit~\cite{chen2017generic} is a generic and efficient method for detecting subspace clusters in attributed graphs. The main idea is to iteratively identify the intermediate solution that is close-to-optimal and then project it to the feasible space defined by the topological and sparsity constraints.
\item UNCut~\cite{ye2017attributed} proposes unimodal normalized cut to find cohesive clusters in attributed graphs. The homogeneity of attributes is measured by the proposed unimodality compactness which also exploits Hartigans' dip-test.
\item AMEN~\cite{perozzi2016scalable,perozzi2018discovering} develops a measure called \textsc{Normality} to quantify both internal consistency and external separability of a graph cluster. Then, the graph cluster that has the best \textsc{Normality} score is extracted.
\item AGC~\cite{zhang2019attributed} is an adaptive graph convolution method for attributed graph clustering, using spectral convolution filters on the vertex attributes.
\item HK~\cite{kloster2014heat} is a local and deterministic method to accurately compute a heat kernel diffusion in a graph. Then, it finds small conductance community around a given seed vertex. HK only considers the graph structure.
\end{itemize} 

We use the Normalized Mutual Information (NMI)~\cite{manning2010introduction} and the $F_1$ score~\cite{van2016local,kloster2014heat} to evaluate the cluster quality. NMI is a widely used metric for computing clustering accuracy of a method against the desired ground truth. NMI is defined as $NMI(\mathcal{C}^\ast,\mathcal{C})=\frac{2\times I(\mathcal{C}^\ast; \mathcal{C})}{H(\mathcal{C}^\ast)+H(\mathcal{C})}$, where $\mathcal{C}^\ast$ is ground truth, $\mathcal{C}$ is the detected cluster, $I(\cdot;\cdot)$ is mutual information, $H(\cdot)$ is entropy. $F_1$ score is the harmonic mean of precision $P$ and recall $R$ and is defined as $F_1=2\cdot\frac{P \cdot R}{P+R}$, where $P =\frac{|\mathcal{C}\cap\mathcal{C}^\ast|}{|\mathcal{C}|}$, $R = \frac{|\mathcal{C}\cap\mathcal{C}^\ast|}{|\mathcal{C}^\ast|}$. The higher the NMI and $F_1$ score, the better the clustering.

For the experiments on the synthetic graphs, we give the correct number of clusters to SG-Pursuit, UNCut, and AGC. We also give the correct size of each cluster to SG-Pursuit. We compute the NMI and the $F_1$ score for each combination of the detected cluster and the ground truth cluster and report the best NMI and $F_1$ score. Note that the selection of the seed vertex and the indexes of the designated attributes depend on user's preferences. In the experiments, we randomly sample a vertex as the seed vertex, and only dip over the most multimodal attribute whose dip test value is the highest. FocusCO needs to compute the relevant attribute weight vector $\beta$ which is then used to weight each edge in the graph. For a fair comparison, the entry in $\beta$ that corresponds to the attribute whose dip test value is the highest is set to one and the other entries are set to zero. AMEN also needs to compute the relevant attribute weight vector. Analogous to FocusCO, we set the corresponding entry to one and other entries zero. Since HK is designed for plain graphs and cannot handle attribute information, we incorporate the attribute information by weighing the edges of the graph using the weighting vector $\beta$. We also report the results of HK on the graph structure. We call these two versions as weighted HK (w) and unweighted HK (uw). We run each experiment 50 times and at each time we randomly sample a seed vertex.

All the experiments are run on the same machine with the Ubuntu 18.04.1 LTS operating system and an Intel Core Quad i7-3770 with 3.4 GHz and 32 GB RAM. LOCLU is written in Java. The code of LOCLU and all the synthetic and real-world graphs used in this work are publicly available at Github\footnote{\url{https://github.com/yeweiysh/LOCLU}}.

\subsection{Synthetic Graphs}
\subsubsection{Clustering Quality}
To study the clustering performance, we generate synthetic graphs with varying numbers of vertices $n$, attributes $d$, varying ratio of relevant attribute and variable cluster size range. For the case of varying $n$, we fix the attribute dimension $d=20$ and the ratio of relevant attributes $50\%$. For the case of varying $d$, we fix the number of vertices $n=1000$ and the ratio of relevant attributes $50\%$.  For the case of varying the ratio of relevant attribute, we fix the attribute dimension $d=20$ and the number of vertices $n=1000$. For varying the cluster size range, we fix the attribute dimension $d=20$ and the ratio of relevant attributes $50\%$. 

All the graphs are generated based on the planted partitions model~\cite{condon2001algorithms} which is also used in FocusCO and SG-Pursuit. Given the desired number of vertices in each cluster, we define a block for the cluster on the diagonal of the adjacency matrix and randomly asign a 1 (an edge) for each entry in the block with probability 0.35 (density of edges in each cluster). For the blocks that are not on the diagonal of the adjacency matrix, we randomly assign an edge for each entry in the block with a probability of 0.01 (density of edges between clusters). We further bisect each graph cluster into two new clusters and then assign attributes to each new cluster. In this case, a method that is only applicable for graph structure cannot detect the ``real cluster'' (unimodal both in the graph structure and designated attributes). To add vertex attributes, for each new graph cluster, we generate the values of relevant attributes according to a Gaussian distribution with the mean value of each attribute randomly sampled from the range $\left[ 0,10\right] $, and the variance value of each attribute 0.001. Following FocusCO~\cite{DBLP:conf/kdd/PerozziASM14}, the variance is specifically chosen to be small such that the clustered vertices ``agree'' on their relevant attributes. To make the other attributes of clusters irrelevant to the graph structure, we first randomly permute the vertex labels and then generate each cluster's irrelevant attribute values according to a Gaussian distribution with mean randomly sampled from the range $\left[ 10,20\right] $ and variance 1.

To study how the number of designated attributes affect the performance of LOCLU, we use our generative model to generate a synthetic data with $n=1000$ vertices, $d=20$ attributes, and the ratio of relevant attributes $50\%$. Figure \ref{fig:varying_deg} shows that LOCLU can almost detect the ground truth. When increasing the number of designated attributes, the performance of LOCLU does not change much. LOCLU considers the data points that situate in the modal interval as the cluster. However, for some boundary data points of Gaussian clusters, the modal interval may not include them. This is the reason why the performance curves of LOCLU have some small vibrations.
\begin{figure}[!htp]
	\centering
	\includegraphics[width=0.4\textwidth]{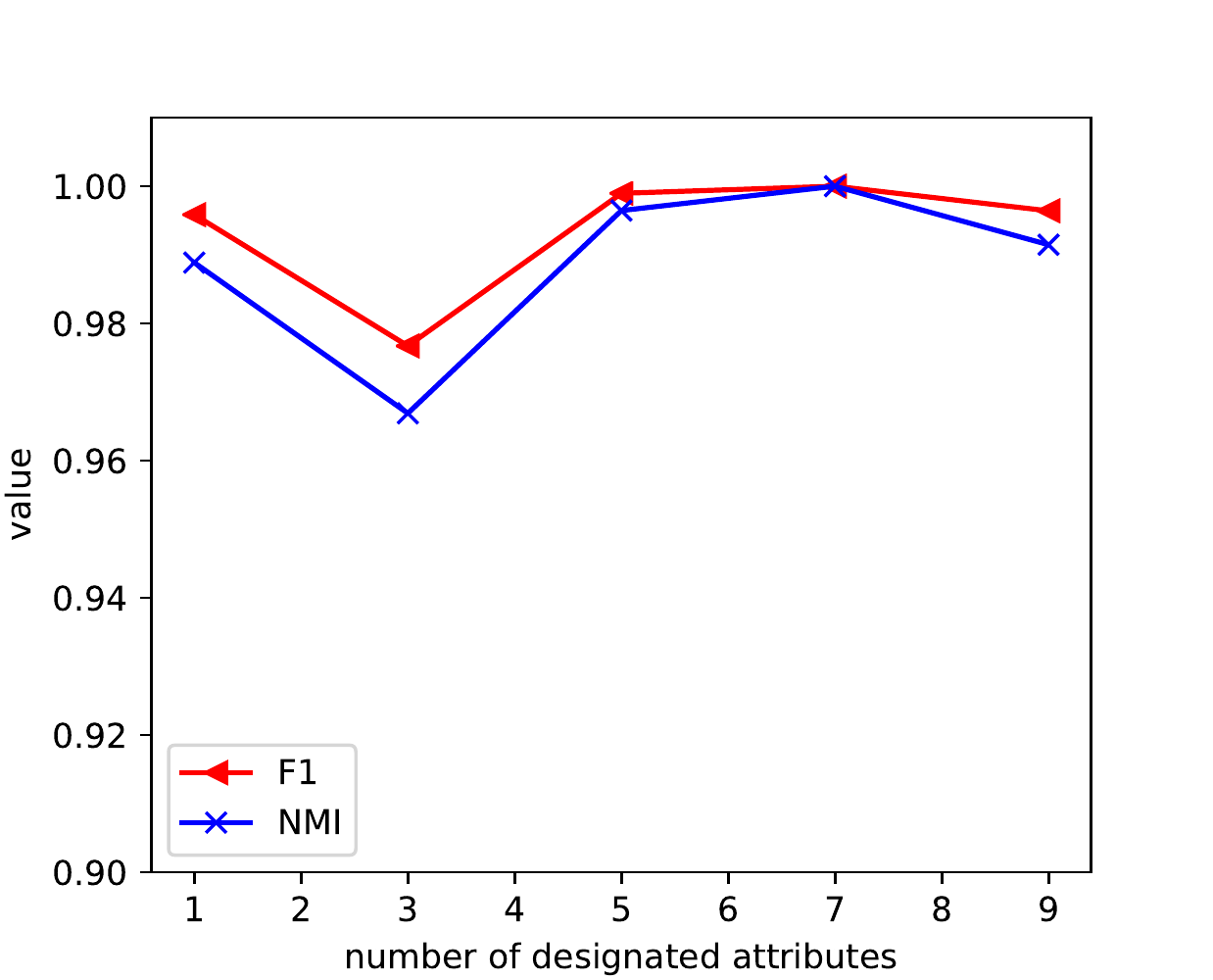}
	\caption{Clustering results of LOCLU with the increasing number of designated attributes.}
	\label{fig:varying_deg}
\end{figure}

Figures~\ref{fig:F1score} and~\ref{fig:NMI} show the quality results. Since HK(w) and HK (uw) have similar performance, we only show one of them. Figure~\ref{fig:F1score}(d) shows the results of each method when varying the cluster size range. We let the graph contains clusters with variable sizes and increase the variance of the cluster sizes. The cluster size is randomly drawn from the variable ranges. In Figure~\ref{fig:F1score}, we can see that LOCLU outperforms the most comparison methods. In most cases, LOCLU beats all the competitors with a large margin, although we provide them with the correct parameters. Figure~\ref{fig:F1score} also shows that SG-Pursuit is the most unstable method compared with the other methods in all these scenarios. AGC is a deep learning method. We can see that AGC is the best in all the comparison methods. Figure~\ref{fig:F1score}(c) demonstrates that the performance of AGC is dramatically increasing with the increasing ratio of relevant attribute. In Figure~\ref{fig:NMI}, we have similar conclusions. As pointed out above, for some boundary data points of Gaussian clusters, the modal interval may not include them. Thus, the curves of LOCLU has some small vibrations. In addition, we randomly generate the mean values of the attributes of the graph cluster. If the mean values of the attributes of two graph clusters are very close, the dip test may think these two clusters' attributes follow a unimodal distribution. Therefore, LOCLU cannot separate them. This is another reason that the curves of LOCLU have some small vibrations.
\begin{figure}[!htb]
	\centering
    \begin{subfigure}{.23\textwidth}
	\includegraphics[width=\textwidth]{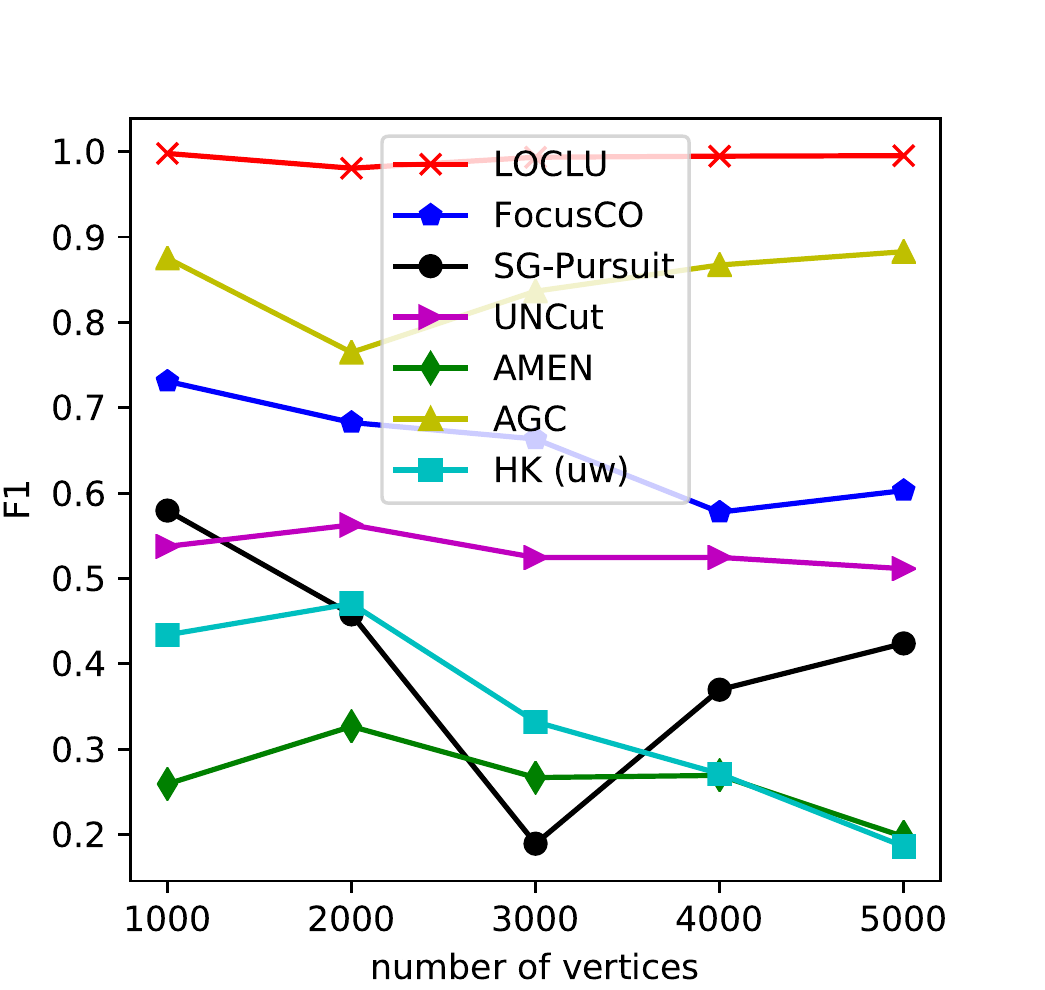}
	\caption{} 
    \end{subfigure}
	\centering
    \begin{subfigure}{.23\textwidth}
	\includegraphics[width=\textwidth]{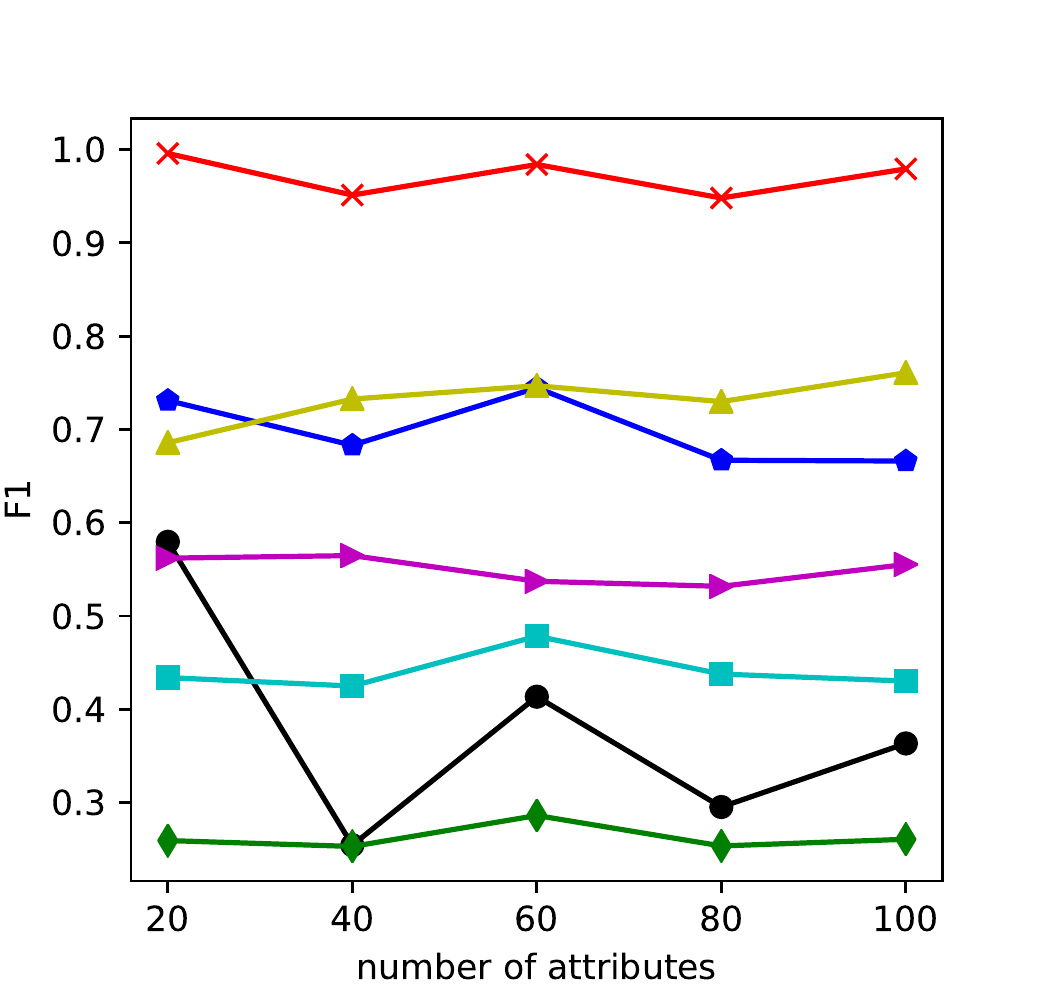}
	\caption{} 
    \end{subfigure}
	
	\centering
	\begin{subfigure}{.23\textwidth}
	\includegraphics[width=\textwidth]{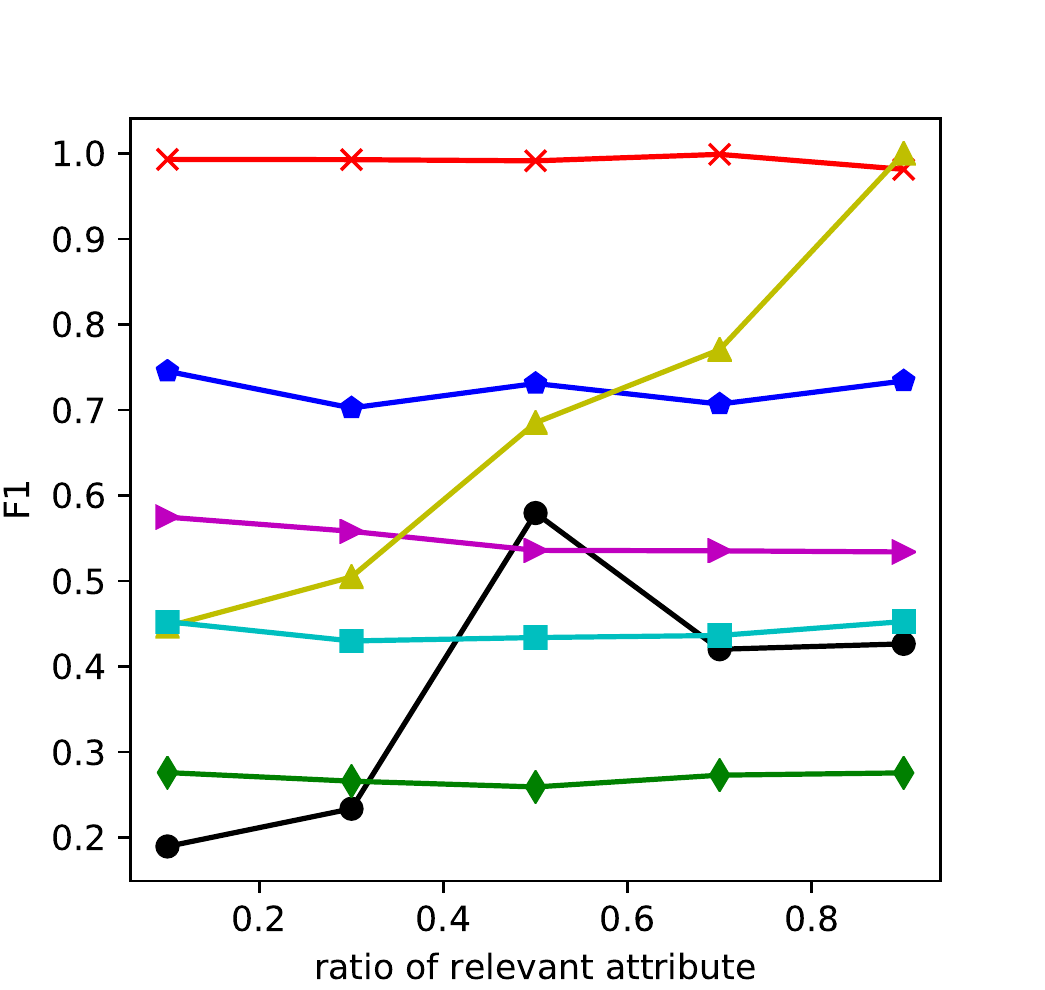}
	\caption{} 
	\end{subfigure}
	\centering
	\begin{subfigure}{.23\textwidth}
	\includegraphics[width=\linewidth]{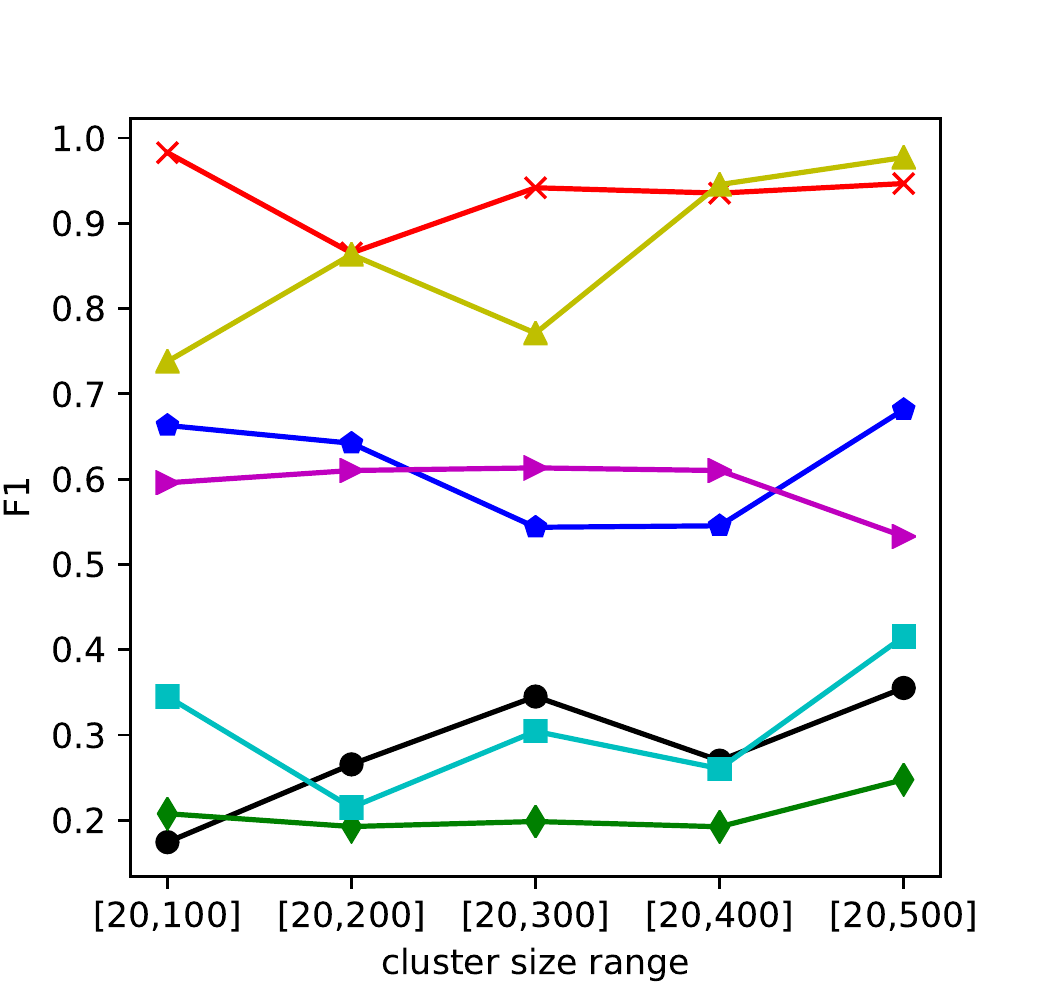}
	\caption{}
	\end{subfigure}
	\caption{Clustering results ($F_1$) on synthetic graphs.}
	\label{fig:F1score}
\end{figure}

\begin{figure}[!htb]
	\centering
	\begin{subfigure}{.23\textwidth}
	\includegraphics[width=\textwidth]{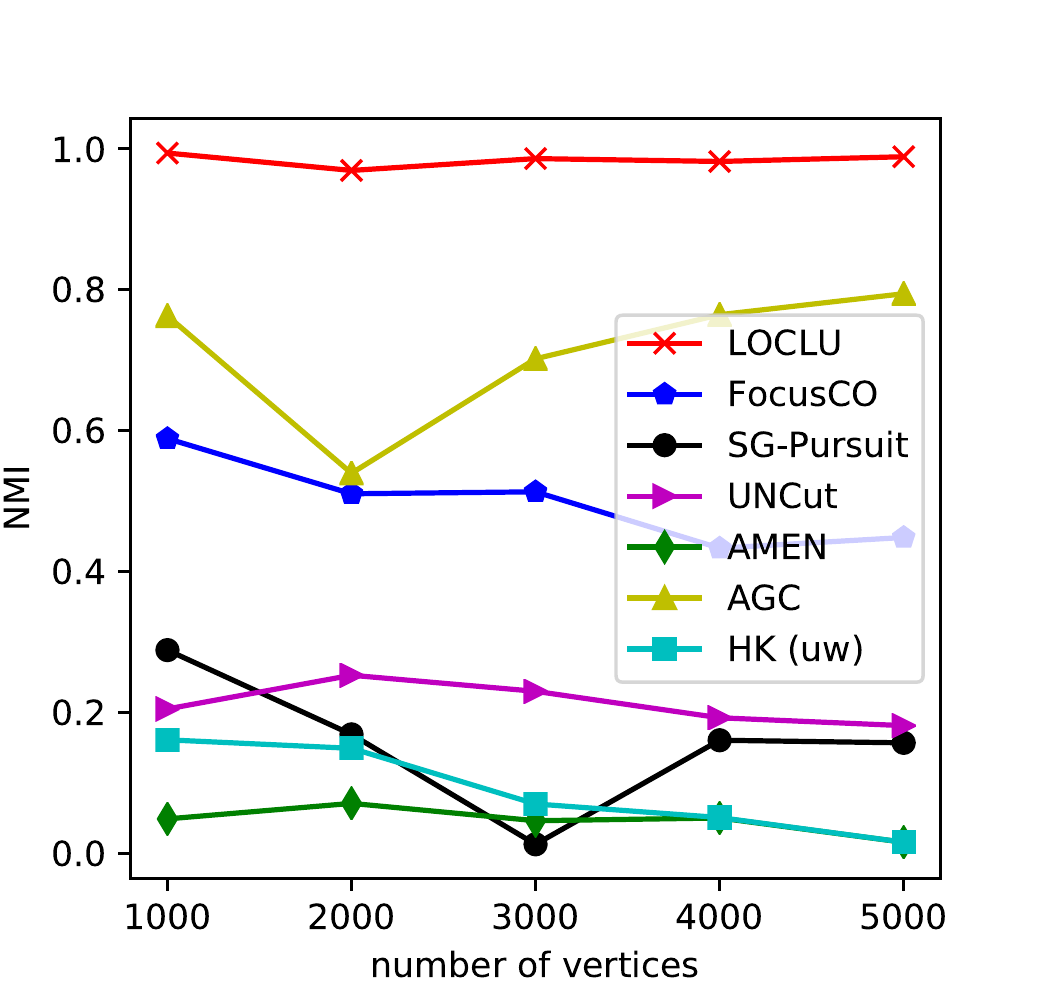}
	\caption{} 
	\end{subfigure}
	\centering
	\begin{subfigure}{.23\textwidth}
	\includegraphics[width=\textwidth]{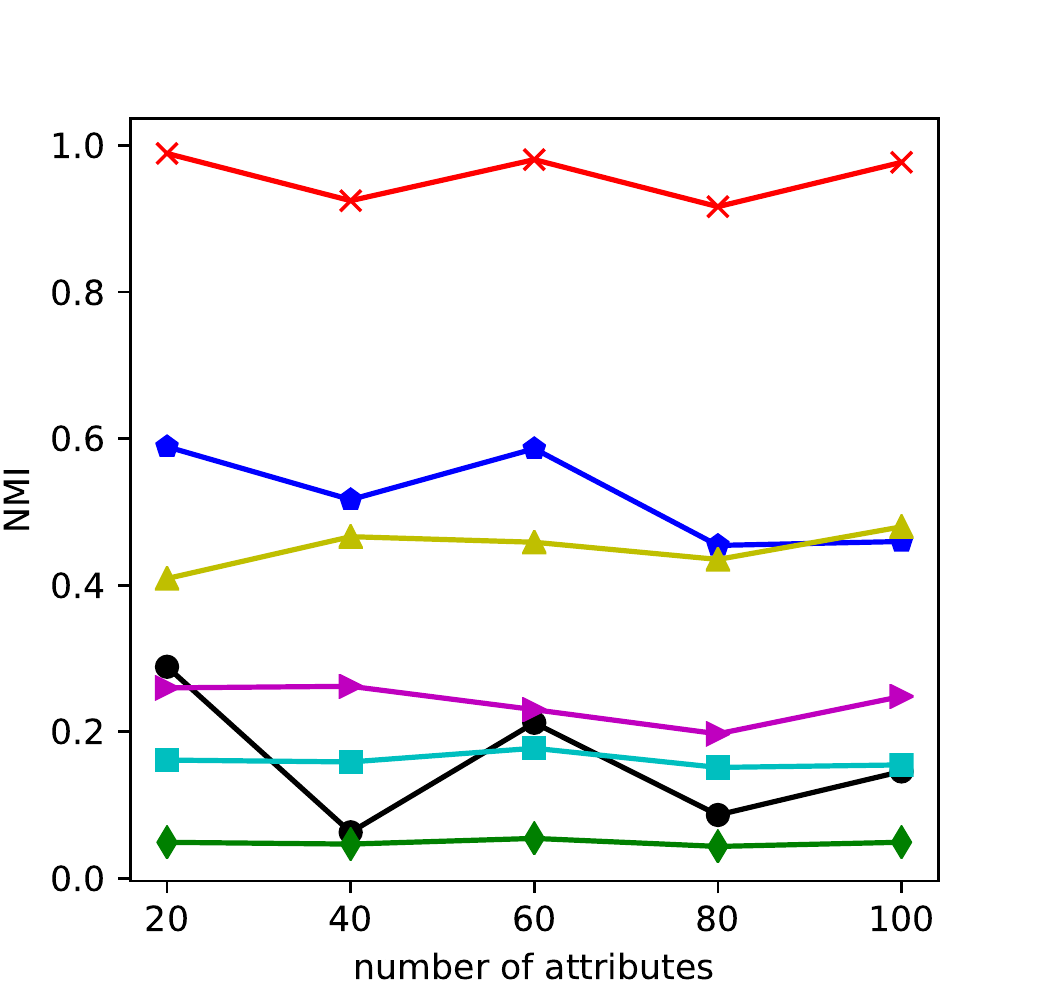}
	\caption{} 
	\end{subfigure}
	
	\centering
	\begin{subfigure}{.23\textwidth}
	\includegraphics[width=\textwidth]{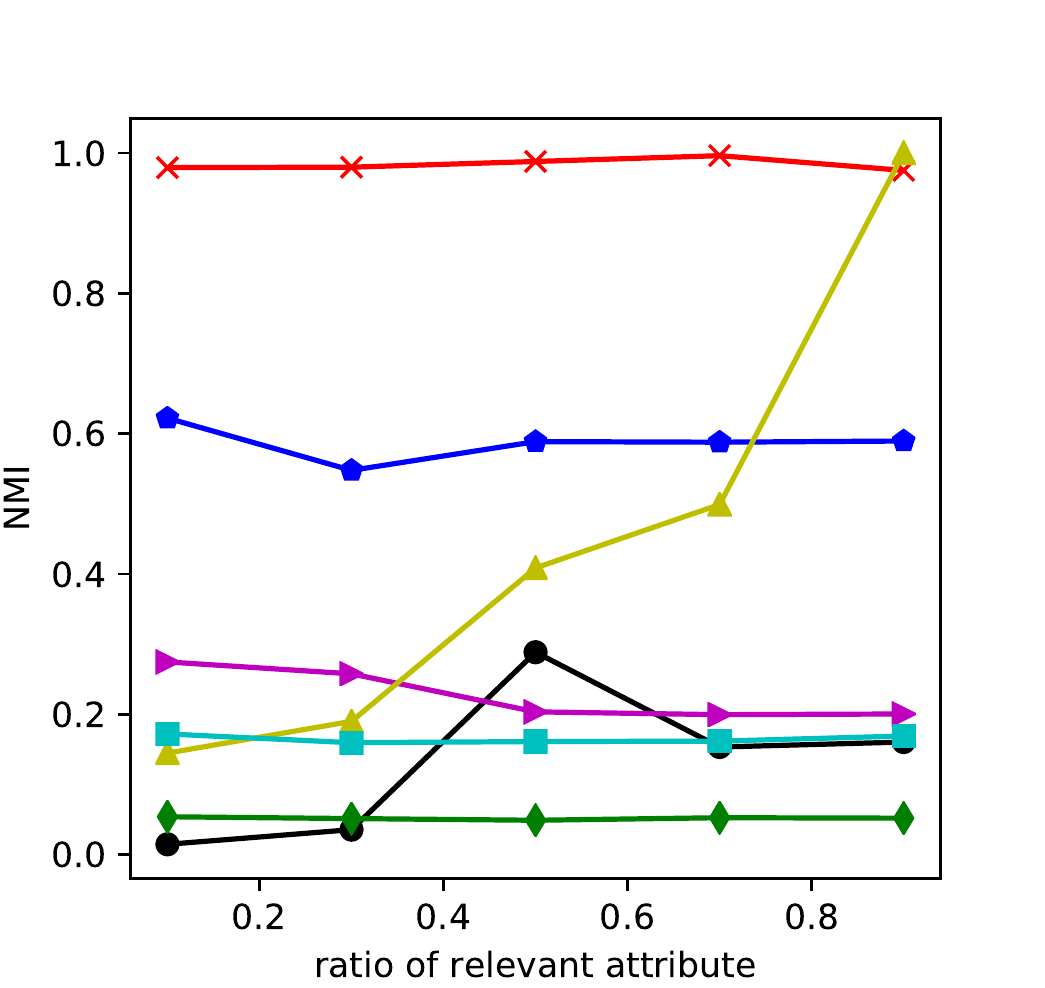}
	\caption{} 
	\end{subfigure}
	\centering
	\begin{subfigure}{.23\textwidth}
	\includegraphics[width=\linewidth]{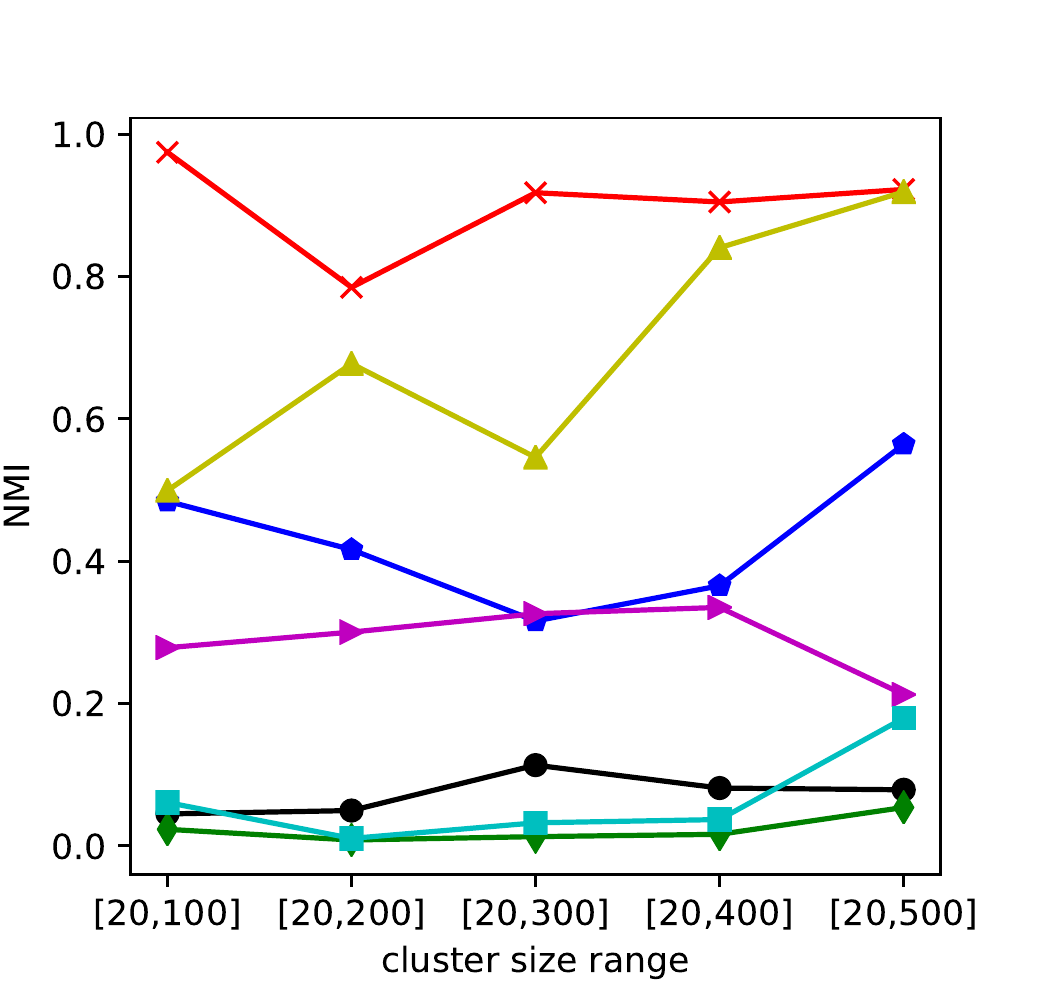}
	\caption{}
	\end{subfigure}
	\caption{Clustering results (NMI) on synthetic graphs.}
	\label{fig:NMI}
\end{figure}

\subsubsection{Scalability}
In this section, we study the scalability of all the methods. We still use the generative model to generate synthetic graphs. For the case of varying the number of attributes, we fix the number of vertices $n=2000$ and the ratio of the relevant attributes $50\%$. For the case of varying the number of vertices, we fix the attribute dimension $d=20$ and the ratio of the relevant attributes $50\%$. Because the running time of the weighted and unweighted versions of the baseline HK are similar, we only give the results of the unweighted version. The runtime of each method is demonstrated in Figure~\ref{fig:runtime}. Since HK (uw) only considers the graph structure, its running time is the lowest. AGC has the second lowest running time. LOCLU outperforms FocusCO, SG-Pursuit, UNCut, and AMEN in most cases.
\begin{figure}[!htb]
	\centering
	\begin{subfigure}{.23\textwidth}
	\includegraphics[width=\linewidth]{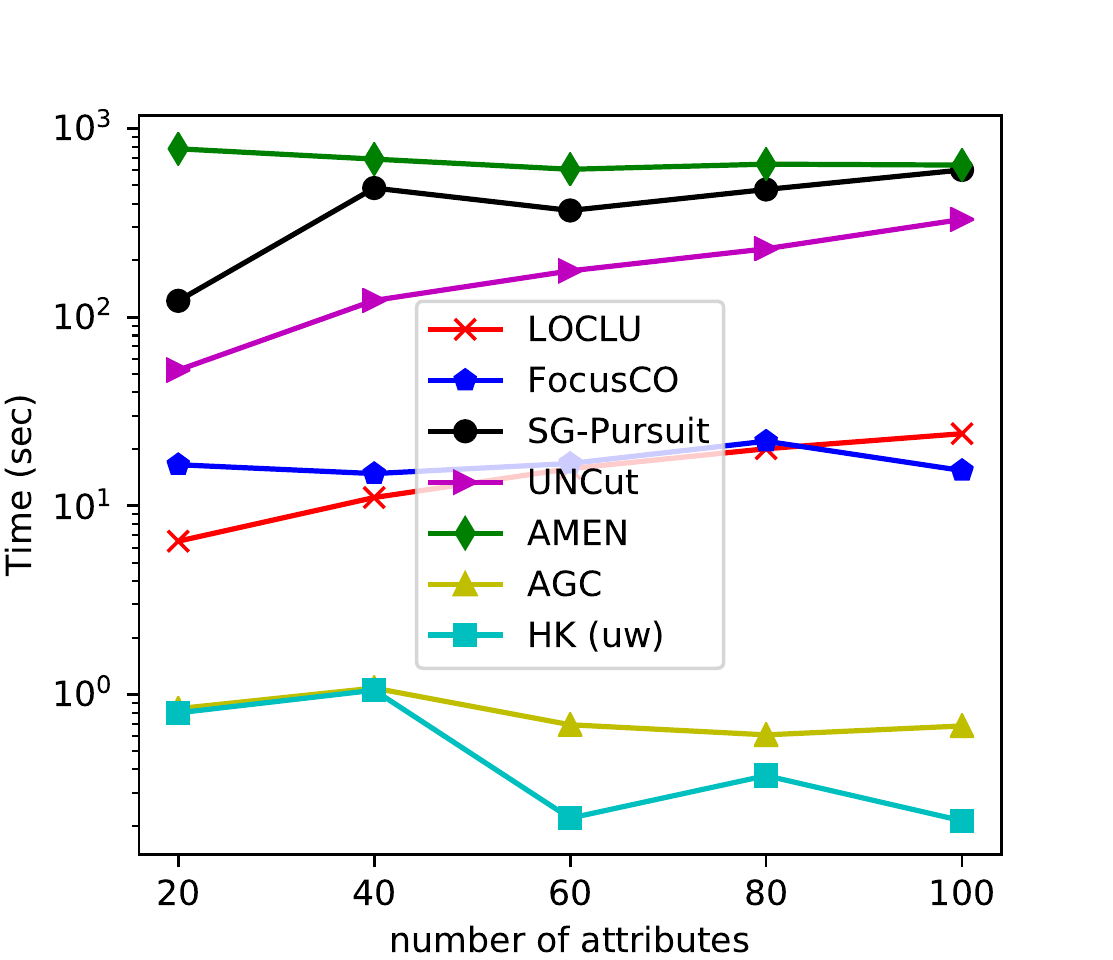}
	\caption{}
	\end{subfigure}
	\centering
	\begin{subfigure}{.23\textwidth}
	\includegraphics[width=\linewidth]{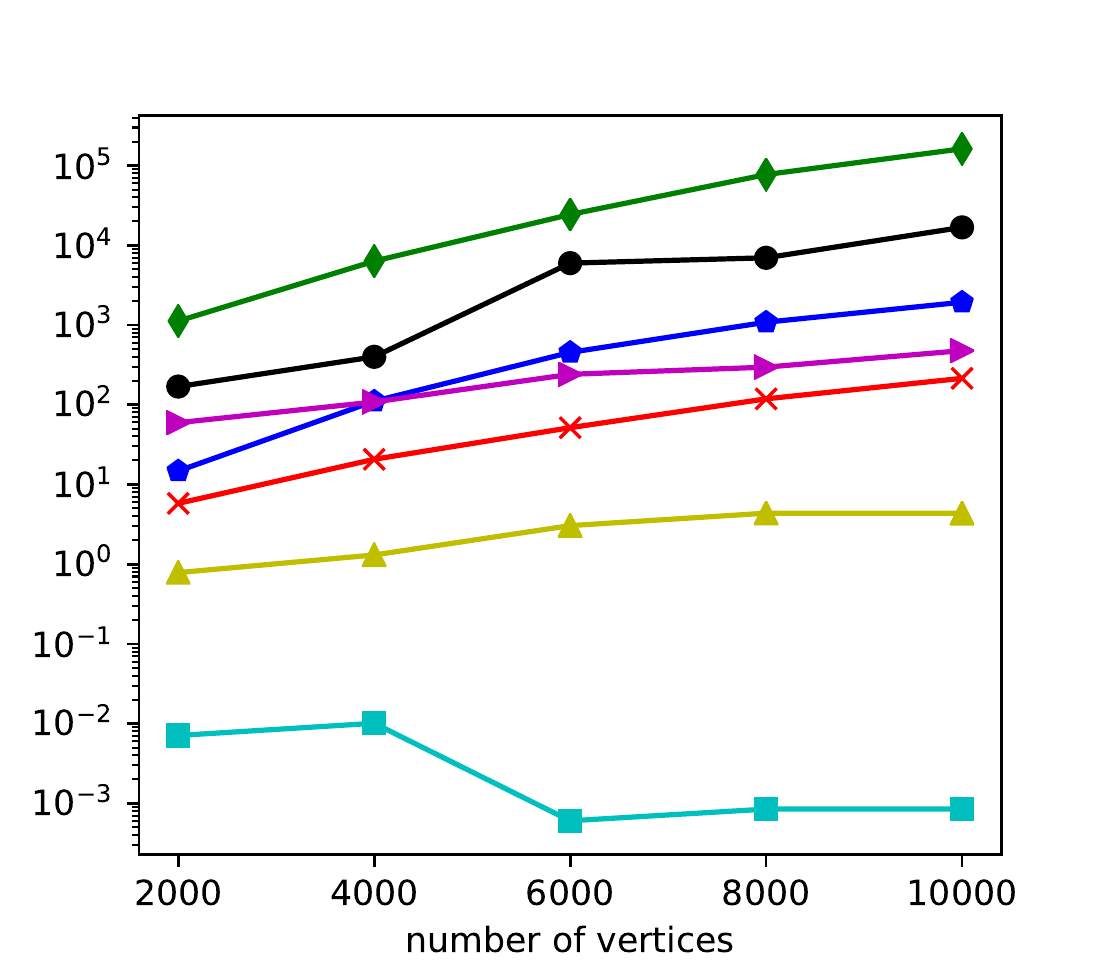}
	\caption{}
	\end{subfigure}
	\caption{Runtime experiments.}
	\label{fig:runtime}
\end{figure}

\subsection{Real-world Graphs} \label{realworld}
We conduct experiments on six real-world attributed graphs whose statistics are given in Table~\ref{tab:sta}. Their details are described in the following. For the real-world graphs, if their attributes are not numeric, i.e., categorical, we use one-hot encoding to tranform the categorical values to numeric ones.
\begin{table}[!htb]
	\centering
	\caption{The statistics of the real-world attributed graphs.}
	\label{tab:sta}
	\begin{tabular}{l|l|l|l|l}
		\toprule
		Datasets                &vertex\# &edge\#   &attribute\# &cluster\#\cite{DBLP:conf/icdm/GunnemannFRS13, ye2017attributed}\\ \hline
		\textsc{Disney}     &124           &333          & 28          &9\\
		\textsc{4Area}      & 26,144       & 108,550    & 4     &50\\ 
		\textsc{ARXIV}   & 856           & 2,660       & 30          &19\\ 
		\textsc{IMDb}     & 862            & 4,388        & 21     &30\\ 
		\textsc{Enron}    & 13,533       & 176,967    & 18     &40\\
		\textsc{Patents}    & 100,000       & 188,631    & 5     &150\\
		\bottomrule
	\end{tabular}
\end{table}

\begin{itemize}
	\item \textsc{Disney}~\cite{DBLP:conf/icdm/SanchezMLKB13}: This network is the Amazon co-purchase network of Disney movies. The network has 124 vertices and 333 edges. Vertices represent movies and edges represent their co-purchase relationships. Each movie has 28 attributes.
	\item \textsc{4Area}~\cite{DBLP:conf/kdd/PerozziASM14}: This network is a co-authorship network of computer science authors. The attributes represent the relevance scores of the publications of an author to the conferences. The categories of conferences are ``databases'', ``data mining'', ``information retrieval'', and ``machine learning''. The network has 26,114 vertices and 108,550 edges.
	\item \textsc{ARXIV}~\cite{DBLP:conf/icdm/GunnemannFRS13}: This network is a citation network whose vertices represent papers and edges represent citation relationships. Attributes denote how often a specific keyword appears in the abstract of the paper. The network has 856 vertices, 2,660 edges, and 30 attributes.
	\item \textsc{IMDb}~\cite{DBLP:conf/icdm/GunnemannFRS13}: This network is extracted from Internet Movie Database. Each vertex represents a movie with at least 200 rankings and an average ranking of at least 6.5. Two movies are connected if they have the same actors. Attributes denote 21 movie genres. The network has 862 vertices and 4,388 edges.
	\item \textsc{Enron}~\cite{DBLP:conf/icdm/SanchezMLKB13}: This network is the communication network with email transmission as edges between email addresses. Each vertex has 18 attributes which describe aggregated information about average content length, average number of recipients, or time range between two mails. The network has 13,533 vertices and 176,967 edges.
	\item \textsc{Patents} \cite{DBLP:conf/icdm/GunnemannFRS13}: This network is a citation network of patents with 100,000 vertices, 188,631 edges, and five attributes which are ``assignee code'', ``claims'', ``patent class'', ``year'' and ``country''.
\end{itemize} 

Since the real-world graphs do not have a ground truth, we use the proposed \textsc{AU}, \textsc{GU}, and \textsc{Compactness} as cluster quality measures. The lower the scores of these three measures, the higher the cluster quality. In addition to these three measures, we also report the \textsc{Normality}~\cite{perozzi2016scalable,perozzi2018discovering} score. The \textsc{Normality} score is a generalization of Newman's modularity and assortativity~\cite{newman2006modularity,newman2002assortative} to attributed graphs. \textsc{Normality} measures both the internal consistency and external separability of an attributed graph cluster. The higher the \textsc{Normality} score, the better the cluster quality. Note that the \textsc{Normality} score can be negative. We give the average scores over 50 runs in Table~\ref{tab:obj} and Table~\ref{tab:normality}, each time with a randomly sampled vertex as the seed vertex. For SG-Pursuit, UNCut and AGC, we give them the same number of clusters as used in~\cite{DBLP:conf/icdm/GunnemannFRS13, ye2017attributed}. For each seed vertex, we first decide which cluster contains it and then compute the scores of these measures.

We can see from Table~\ref{tab:obj} that LOCLU achieves the best \textsc{AU} and \textsc{Compactness} scores on all six real-world datasets. On the dataset \textsc{Disney}, AMEN achieves the best \textsc{GU} score. Table~\ref{tab:normality} shows the \textsc{Normality} score of each method. We can see that LOCLU has the best \textsc{Normality} score on four datasets. On dataset \textsc{Disney}, UNCut has the best \textsc{Normality} score. On dataset \textsc{Patents}, FocusCO has the best \textsc{Normality} score. For case studies, we interprete the results of LOCLU and its competitor FocusCO on \textsc{Disney} and \textsc{4Area} datasets. For FocusCO, we set the entries in $\beta$ that correspond to the designated attributes to one and other entries to zero.
\begin{table*}[!htb]
	\centering
	\caption{The \textsc{AU}/\textsc{GU}/\textsc{Compactness} scores of each method on the real-world attributed graphs. N/A means the results are not available because the method: 1) is not applicable on the unconnected graphs, 2) runs out of memory, or 3) does not finish in a week.}
	\label{tab:obj}
	\begin{tabular}{l|l|l|l|l|l|l}
		\toprule
		Algorithms     &\textsc{Disney}           & \textsc{4Area}                 & \textsc{ARXIV}      &\textsc{IMDb}           &\textsc{Enron} & \textsc{Patents}\\ \hline
		LOCLU    &\textbf{0.010}/0.062/\textbf{0.072} &\textbf{0.002}/\textbf{0.009}/\textbf{0.011}         &\textbf{0}/\textbf{0.027}/\textbf{0.027}    &\textbf{0}/\textbf{0.015}/\textbf{0.015}       &\textbf{0.006}/\textbf{0.001}/\textbf{0.007} &\textbf{0}/\textbf{0.001}/\textbf{0.001}\\ 
		FocusCO         &0.012/0.080/0.092 &0.021/0.075/0.096        &0.074/0.055/0.129    &0.024/0.079/0.103 &0.088/0.001/0.089 &0.012/0.067/0.079\\ 
		SG-Pursuit      &0.088/0.087/0.167  & N/A                                   &0.075/0.074/0.149              &0.051/0.058/0.109  &0.126/0.001/0.127 & N/A   \\ 
		UNCut              &0.094/0.079/0.173  &0.175/\textbf{0.009}/0.184      &0.172/0.051/0.223 &0.148 / 0.023 / 0.171  & 0.149/0/0.149 & 0.162/0.012/0.174  \\ 
		AMEN              &0.071/\textbf{0.055}/0.126   &0.022/0.033/0.055  &N/A           &0.106/0.060/0.166              &N/A &N/A\\
		AGC                 &0.138/0.094/0.232  &0.008/0.011/0.019  &0.138/0.057/0.195  &0.132/0.029/0.161 &0.156/0.001/0.157 &0.007/0.007/0.014\\
		HK (uw)          &0.123/0.070/0.193   &0.050/0.078/0.128  &0.168/0.050/0.218 &0.143/0.019/0.162 &0.167/0.002/0.169 &0.027/0.031/0.058\\
		HK (w)            &0.123/0.070/0.193   &0.050/0.078/0.128  &0.168/0.050/0.218 &0.143/0.019/0.162 &0.167/0.002/0.169 &0.027/0.031/0.058\\ 
		\bottomrule
	\end{tabular}
\end{table*}

\begin{table*}[!htb]
	\centering
	\caption{The \textsc{Normality} score of each method on the real-world attributed graphs. N/A means the results are not available because the method: 1) is not applicable on the unconnected graphs, 2) runs out of memory, or 3) does not finish in a week.}
	\label{tab:normality}
	\begin{tabular}{l|l|l|l|l|l|l}
		\toprule
		Algorithms     &\textsc{Disney}    & \textsc{4Area}        & \textsc{ARXIV}    &\textsc{IMDb}      &\textsc{Enron} & \textsc{Patents}\\ \hline
		LOCLU    &-1.326                    &\textbf{-0.420}         &\textbf{-0.567}                 &\textbf{0.013}       &\textbf{-0.800}  &-1.001\\ 
		FocusCO         &-1.567                   &-0.760                        &-0.832                  &-0.979                  &-1.000 &\textbf{-0.920}\\ 
		SG-Pursuit      &-1.898                   &N/A                             &-0.914                  &-0.979                  &-0.999 &N/A   \\ 
		UNCut              &\textbf{-1.182}     & -1.000                        &-0.999                 &-0.996                   &-1.000 &-1.001  \\ 
		AMEN              &-2.403                  &-0.974            &N/A                   &-0.990              &N/A &N/A\\
		AGC                 &-1.235                    &-1.000                          &-0.987                &-0.998                    &-1.000 &-1.000\\
		HK (uw)          &-1.687                    &-0.780                         &-0.858                &-0.958                     &-0.926 &-1.001\\
		HK (w)            &-1.687                     &-0.780                        &-0.858                 &-0.958                    &-0.926 &-1.001\\ 
		\bottomrule
	\end{tabular}
\end{table*}

  \textbf{Disney}. \textsc{Disney} is a subgraph of the Amazon co-purchase network. Each movie (vertex) is described by 28 attributes, such as ``Average Vote'', ``Product Group'', and ``Price''. Given the seed vertex and one designated attributes ``Amazon Price'', we want to find a local cluster concentrating on this seed vertex and the designated attribute. All the 15 vertices in Figure~\ref{fig:Disney} are read-along movies that are rated as PG (Parental Guidance Suggested) and attributed as ``Action \& Adventure'', e.g., ``Spy Kids'', ``Inspector Gadget'' and ``Mighty Joe Young''. We show the local clusters detected by LOCLU and its competitor FocusCO in Figure~\ref{fig:Disney}. In Figure~\ref{fig:Disney}, the vertex in red is the given seed vertex, and the vertices in blue are the detected vertices. Figure~\ref{fig:Disney}(a) shows the local cluster detected by LOCLU. The \textsc{GU} score is 0.087 and the \textsc{AU} score is 0.110. The \textsc{Compactness} score is 0.197. The \textsc{Normality} score is -1.454. Figure~\ref{fig:Disney}(b) shows the local cluster detected by FocusCO. The \textsc{GU} score is 0.050 and the \textsc{AU} score is 0.100. The \textsc{Compactness} score is 0.150. The \textsc{Normality} score is -1.711. FocusCO is better than LOCLU if considering the \textsc{Compactness} score. LOCLU is superior to FocusCO when considering the \textsc{Normality} score.
  \begin{figure}[!htb]
  	\hspace*{\fill}
  	\centering
  	\begin{subfigure}{.22\textwidth}
  		\fbox{\includegraphics[width=\textwidth]{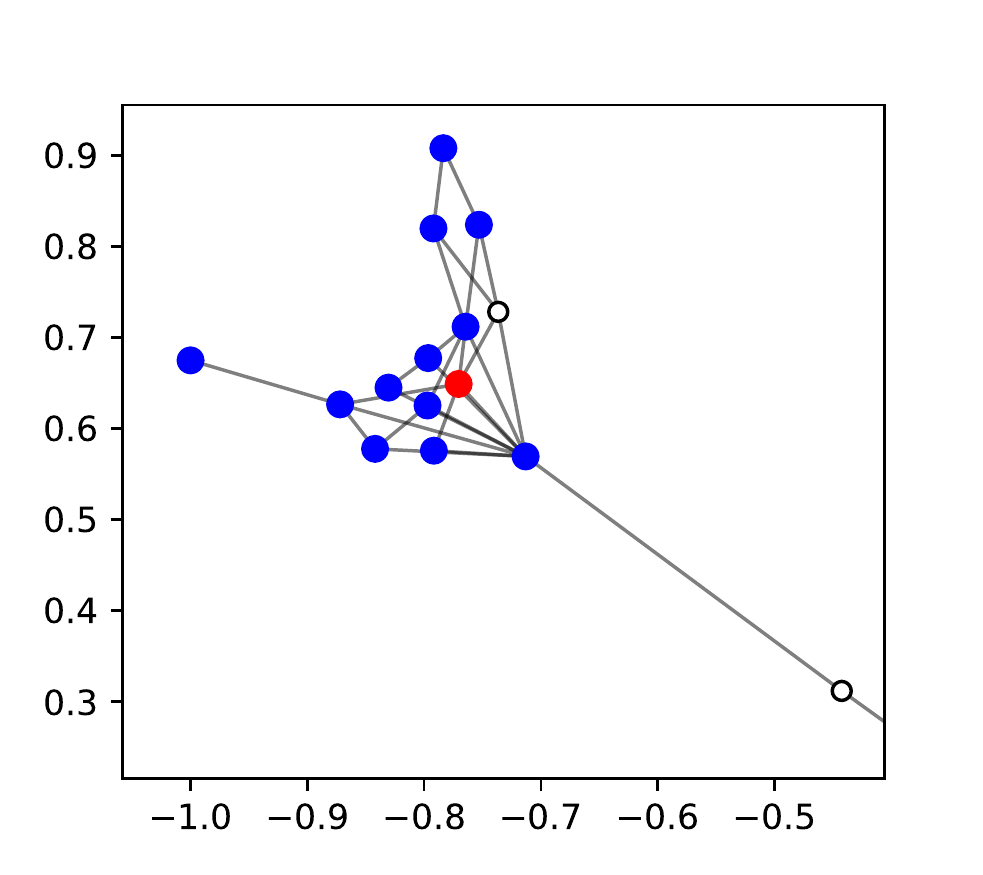}}
  		\caption{LOCLU} 
  	\end{subfigure}
  \hfill
  \hfill
  \centering
  \begin{subfigure}{.22\textwidth}
  	\fbox{\includegraphics[width=\textwidth]{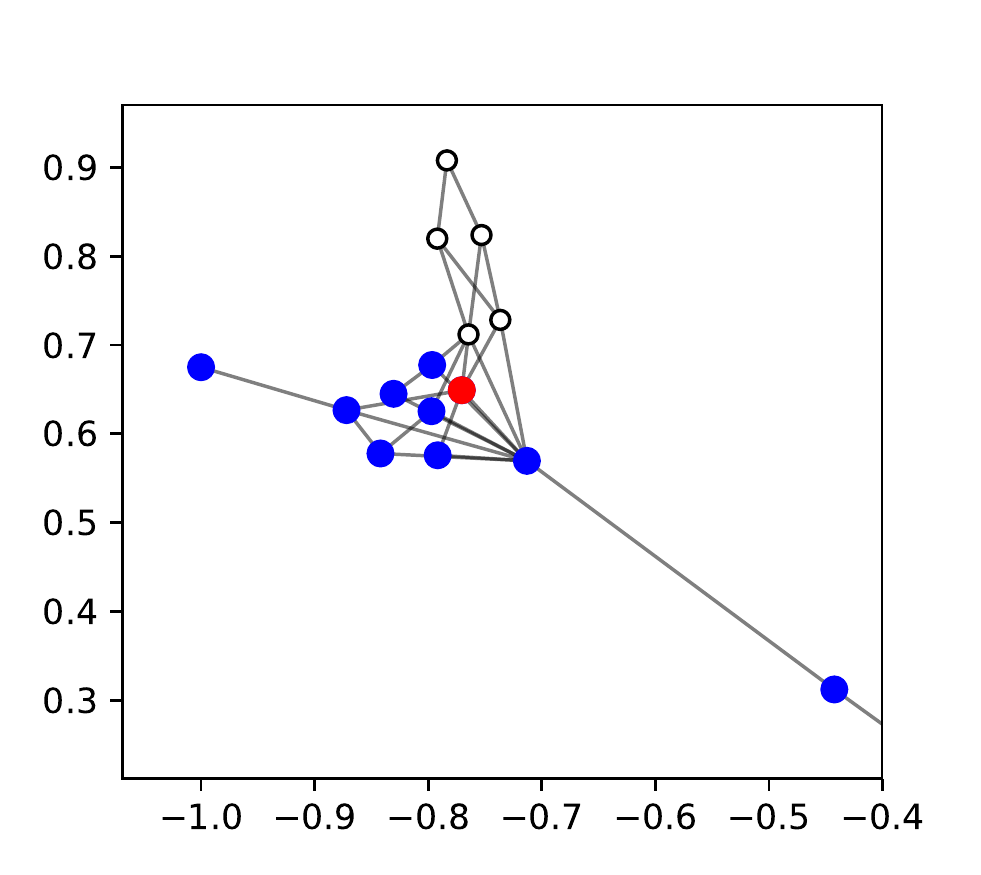}}
  	\caption{FocusCO} 
  \end{subfigure}
 \hspace*{\fill}
  	\caption{Local clusters found in the \textsc{Disney} dataset by LOCLU and FocusCO.}
  	\label{fig:Disney}
  \end{figure}

 \textbf{4Area}. \textsc{4Area} is a co-authorship network of computer science authors. The attributes represent the relevance scores of the publications of an author to the conferences ``databases'', ``data mining'', ``information retrieval'', and ``machine learning''. Given the seed vertex \textit{Jiawei Han} and two attributes ``data mining'' and ``machine learning'', we want to find a local cluster concentrating on this seed vertex and the two designated attributes. Figure~\ref{fig:4Area}(a) shows the main component of the local cluster that includes the given seed vertex (in red) and the detected vertices (in blue) by LOCLU. This subgraph has 37 authors and is unimodal in the two designated attributes ``data mining'' and ``machine learning''. This subgraph contains authors that belong to the ``data mining'' field, but not the ``machine learning'' field. It contains authors such as \textit{Jian Pei}, \textit{Philip S. Yu}, \textit{Hans-Peter Kriegel}, and \textit{Christos Faloutsos} who focus primarily on ``data Mining''. The \textsc{GU} score is 0.046 and the \textsc{AU} score is 0.052. The \textsc{Compactness} score is 0.098. The \textsc{Normality} score is -1.424. The detected cluster (shown in Figure~\ref{fig:4Area}(b)) by FocusCO has 134 authors. This local cluster is not unimodal in the designated attribute ``data mining''. The cluster contains authors such as \textit{Chu Xu} and \textit{Liping Wang} who focus primarily on ``information retrieval'', and authors such as \textit{Joseph C. Pemberton} and \textit{Zhao Xing} who focus primarily on ``machine learning''. The \textsc{GU} score is 0.018 and the \textsc{AU} score is 0.110. The \textsc{Compactness} score is 0.128. The \textsc{Normality} score is -2.000. Thus, the subgraph shown in Figure~\ref{fig:4Area}(a) has a higher quality than that shown in Figure~\ref{fig:4Area}(b).
  \begin{figure*}[!htb]
  	\hspace*{\fill}
  	\centering
  	\begin{subfigure}{.45\textwidth}
  		\fbox{\includegraphics[width=\textwidth]{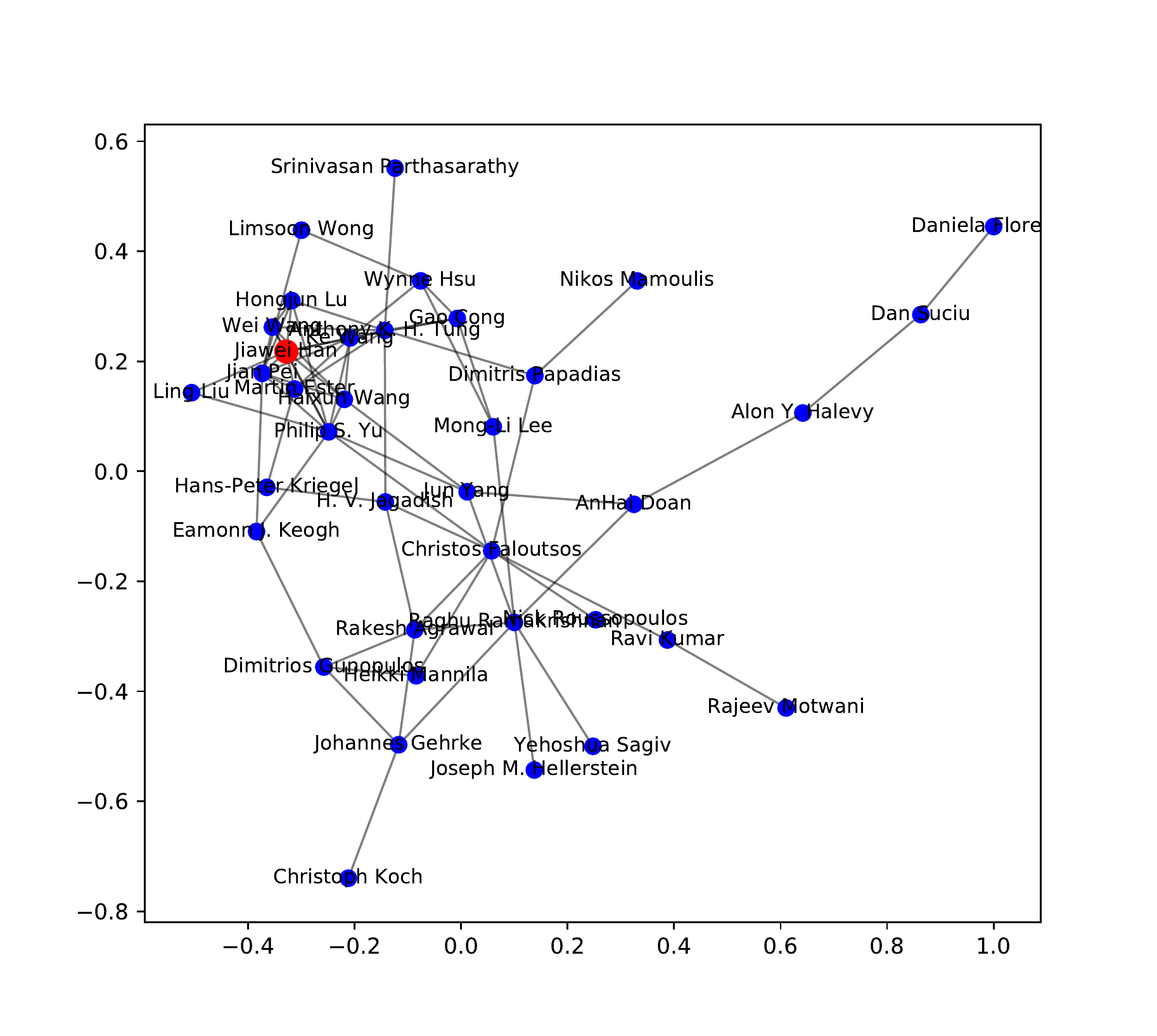}}
  		\caption{LOCLU} 
  	\end{subfigure}
  \hfill
  \hfill
  	\centering
  	\begin{subfigure}{.45\textwidth}
  		\fbox{\includegraphics[width=\textwidth]{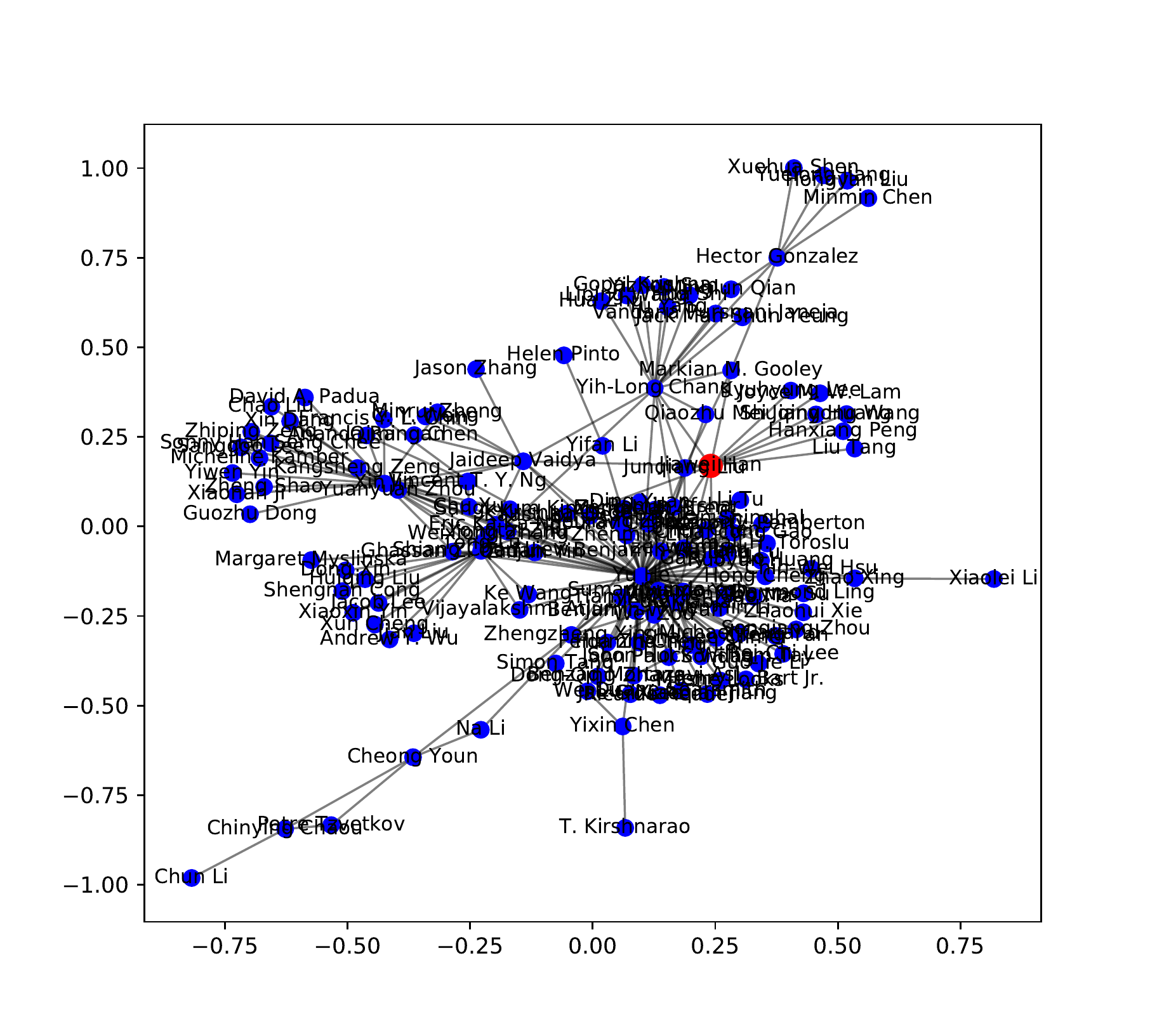}}
  		\caption{FocusCO} 
  	\end{subfigure}
   \hspace*{\fill}
  	\caption{Local clusters found in the \textsc{4Area} dataset by LOCLU and FocusCO. To reduce clutter, we only show a subgraph of \textsc{4Area}, which consists of the seed vertex and the detected vertices.}
  	\label{fig:4Area}
  \end{figure*}

\section{Related Work and Discussion}\label{relatedWork}
\subsection{Plain Graph Clustering}
Plain graphs are those graphs whose vertices have no attributes. Clustering on plain graphs has been well studied in literatures. METIS~\cite{karypis1998multilevel} and spectral clustering~\cite{shi2000normalized, ng2001spectral, ye2016fuse} are typically and widely used methods, which compute a $k$-way partitioning of a graph. METIS is a multi-constraint graph partitioning method, which are based on the multilevel graph partitioning paradigm. Spectral clustering aims to partition the graph into $k$ subgraphs such that the normalized cut criterion is minimized. Instead of optimizing the normalized cut criterion, MODULE~\cite{newman2006modularity} optimizes a quality function known as ``modularity'' over the possible divisions of a graph. The authors showed that the modularity was superior to the normalized cut criterion in the task of community detection. Markov Cluster Algorithm (MCL)~\cite{van2000cluster} is a fast and scalable graph clustering method that is based on simulation of stochastic flow in graphs. Infomap~\cite{rosvall2008maps} is an information theoretic approach that uses the probability flow of random walks on a network as a proxy for information flows and decomposes the network into modules by Minimum Description Length (MDL) principle. Attractor~\cite{shao2015community} automatically detects communities in a network by using the concept of distance dynamics, i.e., the network is treated as an adaptive dynamical system where each vertex interacts with its neighbors. CLMC (Cluster-driven Low-rank Matrix Completion)~\cite{shao2019community} performs community detection and link prediction simultaneously. It first decomposes the adjacency matrix of a graph as three additive matrices: clustering matrix, noise matrix and supplement matrix. Then, the community-structure and low-rank constraints are imposed on the clustering matrix to remove noisy edges between communities.

\subsection{Attributed Graph Clustering}
Differing from the plain graph clustering which groups vertices only taking the graph structure into account, attributed graph clustering achieves detecting clusters in which the vertices have dense edge connectivity and homogeneous attribute values. PICS~\cite{akoglu2012pics} exploits the MDL principle to automatically decide the parameters to detect meaningful and insightful patterns in attributed graphs. SA-Cluster~\cite{DBLP:journals/pvldb/ZhouCY09} first designs a unified neighborhood random walk distance to measure the vertex similarity on an augmented graph. It then uses $k$-medoids to partition the graph into clusters with cohesive intra-cluster structures and homogeneous attribute values. BAGC~\cite{xu2012model} develops a Bayesian probabilistic model for attributed graphs. Clustering on attributed graphs is transformed into a probabilistic inference problem, which is then solveld by an efficient variational method.

The above methods consider all attributes for clustering. However, the irrelevant attributes may be contradicting with the graph structure. In this case, clusters only exist in the subset (subspace) of attributes. For the subspace clustering in attributed graphs, some methods have been proposed. SSCG~\cite{DBLP:conf/icdm/GunnemannFRS13} proposes Minimum Normalized Subspace Cut and detects an individual set of relevant features for each cluster. It needs to update the subspace dependent weight matrix in every iteration, which is very time-consuming. CDE~\cite{li2018community} formulates the community detection in attributed graphs as a nonnegative matrix factorization problem. It first develops a structural embedding method for the graph structure. Then, it integrates community structure embedding matrix and vertex attribute matrix for subsequent nonnegative matrix factorization. CDE is only applicable on graphs with nonnegative vertex attributes.

UNCut~\cite{ye2017attributed} proposes unimodal normalized cut to find cohesive clusters in attributed graphs. The detected cohesive clusters have densely connected edges and have as many homogeneous (unimodal) attributes as possible. The homogeneity or unimodality of attributes is measured by the proposed unimodality compactness which also exploits Hartigans' dip test. The dip test used in UNCut is to measure the unimodality of each attribute. However, in our method LOCLU, the dip test is used to generate modal interval on which the local clustering technique is based. SG-Pursuit~\cite{chen2017generic} is a generic and efficient method for detecting subspace clusters in attributed graphs. The main idea is to iteratively identify the intermediate solution that is close-to-optimal and then project it to the feasible space defined by the topological and sparsity constraints. SG-Pursuit needs to specify the parameters such as the maximum number of vertices in the subspace cluster and the maximum size of selected features which are difficult to set in the real-world datasets. 

Recently, deep learning techniques are adopted for attributed graph clustering. DAEGC~\cite{wang2019attributed} develops a graph attention-based autoencoder to effectively integrate both structure and attribute information for deep latent representation learning. Furthermore, soft labels for the graph representation are generated to supervise a self-training clustering process. The graph representation and self-training processes are unified in one framework. AGC~\cite{zhang2019attributed} is an adaptive graph convolution method for attributed graph clustering. AGC first designs a $k$-order graph convolution that acts as a low-pass graph filter on vertex attributes to obtain smooth feature representations. Then, it utilizes spectral clustering to find clusters in the representation space.

Another research trend is to integrate anomaly detection into the clustering process. AMEN~\cite{perozzi2016scalable,perozzi2018discovering} proposes  a new quality measure called \textsc{Normality} for attributed neighborhoods, which utilizes the graph structure and attributes together to quantify both internal consistency and external separability. \textsc{Normality} is inspired by Newman's modularity and assortativity~\cite{newman2006modularity,newman2002assortative}. Then, a  community and anomaly detection algorithm that uses \textsc{Normality} is proposed to extract communities and anomalies in attributed graphs. Each community is assigned with a few characterizing attributes. PAICAN~\cite{bojchevski2018bayesian} is a probabilistic generative model that jointly models the attribute and graph space, as well as the latent group assignments and anomaly detection. All the methods discussed above need to partition the whole graph structure to find clusters and cannot incorporate user's preference into clustering. 

\subsection{Semi-supervised Graph Clustering}
In many applications, people may be only interested in finding clusters near a target local region in the graph. The methods for plain graph and attributed graph clustering cannot be applied in such a scenario. Several recent methods~\cite{andersen2006communities, leskovec2008statistical, yuchenMRW} focus on using short random walks starting from a small seed set of vertices to find local clusters. There are also some proposals focusing on using the graph diffusion methods to find local clusters, such as PPR~\cite{andersen2006local}, HK~\cite{kloster2014heat}, PGDc~\cite{van2016local}, HOSPLOC~\cite{zhou2017local}, and MAPPR~\cite{yin2017local}. 
PPR~\cite{andersen2006local} is an approximate method to compute the personalized PageRank vector which is used for the local graph partitioning. HK~\cite{kloster2014heat} is a local and deterministic method to accurately compute a heat kernel diffusion in a graph. There are also some methods based on spectral clustering and label propagation for local cluster detection, such as~\cite{he2015detecting, li2015uncovering, mahoney2012local, hansen2014semi, bagrow2005local}. 

However, all these methods are only applicable on the task of local clustering on plain graphs. To the best of our knowledge, there are only two methods focusing on the local clustering on attributed graphs. FocusCO~\cite{DBLP:conf/kdd/PerozziASM14} incorporates user's preference into graph mining and outlier detection. It identifies the relevance of vertex attributes that makes the user-provided examplar vertices similar to each other. Then it reweighs the graph edges and extracts the focused clusters. FocusCO cannot infer the projection vector if the examplar set has only one vertex. LOCLU can find a local cluster around a given seed vertex. If given a set of vertices whose designated attribute values follow a unimodal distribution, LOCLU can also work. However, if their designated attribute values follow a multimodal distribution, LOCLU cannot find a local cluster that includes all these vertices. Like other clustering methods, LOCLU also has limitations. For example, the univariate projection for the dip test may cause information-loss in some cases. TCU-SA (Target Community Detection with User's Preference and Attribute Subspace)~\cite{liu2019target} first computes the similarities between vertices and then expand the query vertex set with its neighbors. Based on the expanded set, TCU-SA deduces the attribute subspace using an entroy method. Finally, the target community is extracted. The idea is very similar to that of FocusCO.

\subsection{Community Search}
Community search over attributed graphs in database research field is also related to our work. Given an input set of query vertices $\mathcal{V}_q$ and their corresponding attributes, find a community containing $\mathcal{V}_q$, in which vertices are densely connected and have homogeneous attributes. These methods include~\cite{huang2015approximate, huang2017attribute, fang2016effective, fang2017effective}. CTC (closest truss community)~\cite{huang2015approximate} is proposed to find a connected $k$-truss subgraph that has the largest $k$, contains $\mathcal{V}_q$, and has the minimum diameter. The problem is NP-hard and the authors develops a greedy algorithm to find a satisfied community. ATC (attribute truss community)~\cite{huang2017attribute} formulates the community search on attributed graphs as finding attributed truss communities. The detected communities are connected and close $k$-truss subgraphs which contains $\mathcal{V}_q$ and has the largest attribute relevance score proposed by the authors. ACQ (attributed community query)~\cite{fang2016effective, fang2017effective} develops the CL-tree index structure and three algorithms based on it for efficient attributed community search. The CL-tree is devised to organize the vertex attribute data in a hierarchical structure. The community search methods are only applicable on categorical attributes. The detected vertices have the same attribute values to those of the query vertices. For continuous attributes, they cannot search a community that is unimodal in the subspace that is composed of the designated attributes. In addition, they are based on dense subgraph structures, such as quasi-clique, $k$-core, or $k$-truss, which are not commonly used in graph clustering.

\section{Conclusion}\label{conclusion}
In this work, we have proposed LOCLU for incorporating user's preference into attributed graph clustering. Currently, very few methods can deal with this kind of task. To achieve the goal, we first propose a new quality measure called \textsc{Compactness} that measures the unimodality of both the graph structure and the subspace that is composed of the designated attributes of a local cluster. Then, we propose LOCLU to optimize the \textsc{Compactness} score. Empirical studies prove that our method LOCLU is superior to the state-of-the-arts. In the future, we will further explore node embeddings for attributed graphs, which seamlessly integrate information from both the attributes and graph structure.

\section*{Acknowledgment}
The authors would like to thank anonymous reviewers for their constructive and helpful comments. This work was supported partially by the U.S. National Science Foundation (grant \# IIS-1817046) and by the U.S. Army Research Laboratory and the U.S. Army Research Office (grant \# W911NF-15-1-0577).

\bibliographystyle{ACM-Reference-Format}
\bibliography{reference}

\end{document}